\documentclass{article}

\usepackage{iclr2027_conference,times}
\usepackage{amsmath,amssymb,amsthm}
\usepackage{graphicx}
\usepackage{tikz}
\usetikzlibrary{arrows.meta,positioning,fit,backgrounds,shapes.geometric}
\usepackage{booktabs}
\usepackage{array}
\usepackage{tabularx}
\usepackage{pifont}
\usepackage{algorithm}
\usepackage{algpseudocode}
\usepackage{microtype}
\usepackage{enumitem}
\usepackage[most]{tcolorbox}
\usepackage{placeins}
\usepackage{hyperref}
\usepackage{url}
\usepackage[capitalize,noabbrev]{cleveref}

\newtheorem{theorem}{Theorem}
\newtheorem{lemma}{Lemma}
\newtheorem{proposition}{Proposition}
\newtheorem{corollary}{Corollary}
\crefname{appendix}{Appendix}{Appendices}
\Crefname{appendix}{Appendix}{Appendices}
\crefname{algorithm}{Algorithm}{Algorithms}
\Crefname{algorithm}{Algorithm}{Algorithms}

\tcolorboxenvironment{proof}{
  enhanced,
  breakable,
  colback=black!1,
  colframe=blue!30!gray,
  boxrule=0.45pt,
  arc=1mm,
  left=1mm,
  right=1mm,
  top=0.8mm,
  bottom=0.8mm,
  before skip=4pt,
  after skip=5pt
}
\newtcolorbox{proofbox}[1][Proof]{
  enhanced,
  breakable,
  colback=black!1,
  colframe=blue!30!gray,
  colbacktitle=blue!7,
  coltitle=black,
  boxrule=0.45pt,
  arc=1mm,
  left=1mm,
  right=1mm,
  top=0.8mm,
  bottom=0.8mm,
  before skip=4pt,
  after skip=5pt,
  fonttitle=\bfseries,
  title={#1}
}

\newcommand{\yh}[1]{}
\newcommand{\YH}[1]{}

\newcommand{\R}{\mathbb{R}}
\newcommand{\fl}{\operatorname{fl}}
\newcommand{\norm}[1]{\left\lVert #1 \right\rVert}
\newcommand{\abs}[1]{\left\lvert #1 \right\rvert}
\newcommand{\ExecCert}{ExecCert}
\newcommand{\RRV}{RRV}
\newcommand{\RRVFull}{Retraining-Reference Release Verification}
\newcommand{\OnDemandRRV}{On-demand \RRV}
\newcommand{\IncrementalRRV}{Incremental \RRV}
\newcommand{\cmark}{\ding{51}}
\newcommand{\xmark}{\ding{55}}

\title{From Mathematical to Executable Certificates for Machine Unlearning}
\author{%
  Ziyu Zhao\thanks{Equal contribution.}$^{,1,2}$
  \quad Xinyu Wang\footnotemark[1]$^{,2}$
  \quad Xiao-Wen Chang$^{2}$
  \quad Yixuan He\thanks{Corresponding author.}$^{,1}$\\[2pt]
  \normalfont\normalsize
  $^{1}$Arizona State University
  \quad
  $^{2}$McGill University%
}

\iclrfinalcopy

\begin{document}

\maketitle
\lhead{Preprint}

\begin{abstract}
Machine unlearning is needed when data must be removed because of deletion
requests, outdated records, or data-quality concerns, while retraining from
scratch can be costly.
Certified machine unlearning methods provide mathematical guarantees, while
deployed systems release concrete finite-precision artifacts produced by
software.
To bridge the gap between mathematical guarantees and practical deployment, we
introduce Executable Release Certification (\ExecCert{}), a release-time layer
that certifies the candidate artifact considered for release. \ExecCert{} either
closes a method's native certificate for the executed candidate or applies
Retraining-Reference Release Verification (\RRV{}) to certify fidelity to current
retain-set retraining. Sequential deletion makes the latter nontrivial because
the exact retain-set reference and the stored numerical state evolve separately.
For frozen representations with a mutable ridge head, we develop an incremental
realization of \RRV{} that maintains certified evidence across deletion requests
rather than reconstructing it at each release.
On four published unlearning implementations, \ExecCert{} preserves valid
certificates, changes release decisions, tightens conservative bounds, and
identifies the retraining-reference fidelity supported by concrete outputs. In
sequential-service experiments, \RRV{} eliminates false releases caused by
stored-equation verification while closely tracking realized error, and
incremental certification remains cheaper than both fresh and maintained
verified-factor alternatives once release checks become sufficiently frequent.
\end{abstract}

\section{Introduction}
\label{sec:introduction}
Machine-learning models may need to remove the influence of training records
when users withdraw consent, data become outdated, or errors are discovered;
regulations such as the EU's right to be forgotten provide an additional
motivation. Retraining from scratch on the retained data gives a natural
reference but can be costly, motivating machine unlearning methods that seek to
approximate retraining at lower deletion cost
\citep{ginart2019,guo2020,sekhari2021,gupta2021}. Certified unlearning further
provides mathematical guarantees relating an efficient deletion mechanism to
this reference
\citep{guo2020,sekhari2021,gupta2021,chien2024,zhang2024certified,
koloskova2025,pandey2026}. However, a deployed system releases a concrete
finite-precision artifact produced by software, rather than the mathematical
object analyzed by the guarantee. This raises a distinct deployment question:
\emph{does the artifact actually considered for release satisfy the intended
certificate or retraining-reference contract?}

Sequential deletion makes this deployment-level question especially important.
After each request, the exact retain-set system changes, while the deployed
service advances a separate finite-precision state through rounded updates.
These two systems do not always coincide: a candidate may solve the stored equation
accurately, sometimes even with zero stored residual, but still violate the contract
defined relative to current retain-set retraining.
Verifying the stored computation alone is therefore insufficient; release-time
certification must connect the executed artifact to the intended reference.

We introduce \emph{Executable Release Certification} (\ExecCert{}) to fill
the gap between the two systems. Before a candidate is released, the deployment fixes a release
contract, and \ExecCert{} certifies whether the concrete executed artifact
satisfies that contract. It supports two paths: when an unlearning method
already exposes a certificate quantity, \emph{native closure} rigorously
encloses that quantity for the executed candidate; when the contract is
defined relative to retraining on the current retain set,
\emph{Retraining-Reference Release Verification} (\RRV{}) instead certifies
the candidate directly against that reference. In both cases, the output
generator, release contract, and verifier remain separate components.

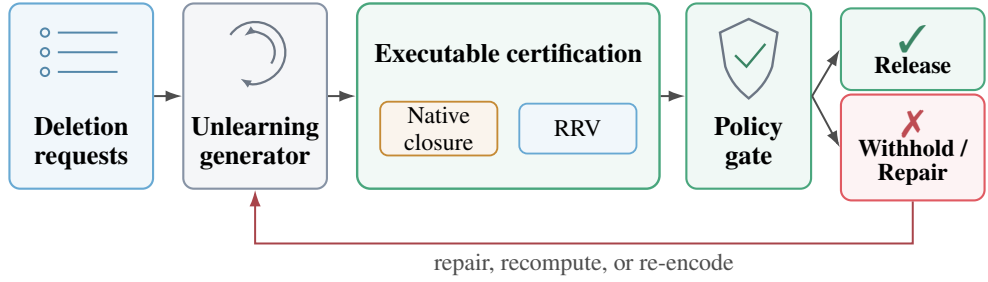
\begin{figure}[tbp]
\centering
\begin{tikzpicture}[
  x=1cm,y=1cm,
  font=\small,
  flow/.style={-{Latex[length=2.2mm,width=1.5mm]}, draw=black!70,
    line width=0.75pt},
  box/.style={rounded corners=4pt, line width=0.8pt, align=center,
    minimum height=2.45cm, inner sep=3pt},
  pathbox/.style={rounded corners=3pt, line width=0.7pt, align=center,
    minimum width=1.55cm, minimum height=0.66cm, inner sep=2pt},
  outcome/.style={rounded corners=4pt, line width=0.8pt, align=center,
    minimum width=1.90cm, minimum height=0.95cm, inner sep=2pt}
]
  \definecolor{figblue}{HTML}{68A9D3}
  \definecolor{figgray}{HTML}{8793A5}
  \definecolor{figgreen}{HTML}{49A67D}
  \definecolor{figorange}{HTML}{C58A35}
  \definecolor{figred}{HTML}{E05A62}

  \node[box, draw=figblue, fill=figblue!9, minimum width=1.90cm]
    (requests) at (0.95,0) {};
  \node[font=\bfseries, align=center] at (0.95,-0.62)
    {Deletion\\requests};
  \begin{scope}[shift={(0.95,0.58)}, draw=figblue!80!black, line width=0.75pt]
    \foreach \yy in {0.26,0,-0.26}{
      \draw (-0.48,\yy) circle (0.055);
      \draw (-0.31,\yy) -- (0.48,\yy);
    }
  \end{scope}

  \node[box, draw=figgray, fill=figgray!8, minimum width=1.90cm]
    (generator) at (3.25,0) {};
  \node[font=\bfseries, align=center] at (3.25,-0.62)
    {Unlearning\\generator};
  \begin{scope}[shift={(3.25,0.58)}, draw=figgray!80!black, line width=0.85pt]
    \draw[-{Latex[length=1.5mm]}] (145:0.40) arc (145:-120:0.40);
    \draw[-{Latex[length=1.5mm]}] (-35:0.26) arc (-35:125:0.26);
  \end{scope}

  \node[draw=figgreen, fill=figgreen!7, rounded corners=5pt,
    line width=0.9pt, minimum width=4.00cm, minimum height=2.45cm]
    (cert) at (6.60,0) {};
  \node[font=\normalsize\bfseries] at (6.60,0.57) {Executable certification};
  \node[pathbox, draw=figorange, fill=figorange!9]
    at (5.68,-0.43) {Native\\closure};
  \node[pathbox, draw=figblue, fill=figblue!9]
    at (7.52,-0.43) {RRV};

  \node[box, draw=figgreen, fill=figgreen!8, minimum width=1.65cm]
    (gate) at (9.78,0) {};
  \node[font=\bfseries, align=center] at (9.78,-0.62)
    {Policy\\gate};
  \begin{scope}[shift={(9.78,0.59)}, draw=figgray!80!black, line width=0.9pt]
    \draw (0,0.43) -- (0.38,0.25) -- (0.31,-0.20)
      .. controls (0.22,-0.45) and (0.08,-0.55) .. (0,-0.62)
      .. controls (-0.08,-0.55) and (-0.22,-0.45) .. (-0.31,-0.20)
      -- (-0.38,0.25) -- cycle;
    \draw[figgreen!80!black] (-0.18,-0.05) -- (-0.03,-0.20) -- (0.22,0.10);
  \end{scope}

  \node[outcome, draw=figgreen, fill=figgreen!8, minimum height=1.08cm]
    (release) at (11.95,0.62)
    {\textcolor{figgreen!80!black}{\Large\ding{51}}\\[-2pt]
     \textbf{Release}};
  \node[outcome, draw=figred, fill=figred!7, minimum height=1.40cm]
    (withhold) at (11.95,-0.68)
    {\textcolor{figred!85!black}{\Large\ding{55}}\\[-2pt]
     \textbf{Withhold /}\\[-1pt]\textbf{Repair}};

  \draw[flow] (requests.east) -- (generator.west);
  \draw[flow] (generator.east) -- (cert.west);
  \draw[flow] (cert.east) -- (gate.west);
  \draw[flow] (gate.east) -- (release.west);
  \draw[flow] (gate.east) -- (withhold.west);

  \draw[flow, figred!75!black] (withhold.south) -- (11.95,-1.95)
    -- node[below, font=\small, text=black!70]
      {repair, recompute, or re-encode} (3.25,-1.95)
    -- (generator.south);
\end{tikzpicture}
\caption{
\textbf{\ExecCert{} certifies a concrete unlearning artifact before release.}
The generator's finite-precision output is checked against a
predeclared contract through native-certificate closure or retraining-reference
verification (RRV). Only certified artifacts pass the policy gate; an
inconclusive candidate is withheld or repaired and then recertified.
}
\label{fig:conceptual-pipeline}
\end{figure}

Our main sequential realization is \emph{\IncrementalRRV{}} for ridge
regression heads with symmetric positive-definite (SPD) normal matrices.
\RRV{} itself certifies only the concrete candidate considered for
release. During the preceding deletion sequence, \IncrementalRRV{} maintains
certified relations between the service's finite-precision statistics and the
exact retain-set statistics, together with reusable inverse-defect evidence.
Here $d$ denotes the feature dimension of each ridge regression head.
Guarded rank-one transitions update this proof state in $O(d^2)$ work.
When a candidate is considered for release, the maintained evidence is
composed into its retraining-reference certificate. Alternatively,
\OnDemandRRV{} constructs the required evidence when certification is
requested.

Experiments on four published unlearning implementations and sequential
services show that \ExecCert{} changes release decisions, tightens or preserves
existing certificates, and characterizes the retraining-reference fidelity
supported by emitted outputs. In sequential services, \RRV{} prevents false
releases admitted by stored-equation checking while closely tracking realized
error; on frozen 3B representations with a mutable ridge head,
\IncrementalRRV{} remains cheaper than both fresh and maintained verified-factor
alternatives once release checks become sufficiently frequent.

\textbf{Contributions.}
We make three contributions.
First, we formalize \emph{executable release certification} as a distinct
layer between a mathematical unlearning mechanism and the concrete artifact
considered for release. \ExecCert{} supports both rigorous closure of a
method's native certificate and certification against current retain-set
retraining
(\cref{sec:problem,sec:common-certificate}).
Second, we develop \IncrementalRRV{} for sequential deletion, maintaining
certified relations between finite-precision service state and the exact
retain-set system together with reusable inverse-defect evidence. Guarded
rank-one transitions update this proof state in \(O(d^2)\) work
(\cref{sec:joint-certificate,sec:policy}).
Third, we evaluate \ExecCert{} on four published unlearning implementations
and sequential services, showing that executable certification has practical
consequences: it can
change release decisions, tighten or preserve existing certificates, prevent
false releases, and amortize verification when certification is frequent
(\cref{sec:experiments}).

\section{Related Work}
\label{sec:related}

Full retraining on the retain set is a natural exact reference but incurs the
cost of processing the surviving data. Exact unlearning targets the same output
distribution and need not literally rerun the original training algorithm;
certified approximate removal instead applies an efficient update and bounds
its deviation from the declared reference, through influence or Newton steps,
randomized mechanisms, or guarantees that compose across a deletion sequence
\citep{guo2020,sekhari2021,gupta2021,chien2024,koloskova2025,
zou2025highdim,pandey2026}. These results differ in what they bound, but share
an assumption: the analysis governs the mathematical update, and the deployed
artifact is taken to realize it. A second line asks instead how an external
party can be convinced that deletion happened, auditing execution traces and
provenance, probing the released model for residual influence of the deleted
records, or attaching cryptographic proofs that a prescribed computation was
carried out
\citep{thudi2022auditable,weng2024poul,eisenhofer2025verifiable,
zhang2024fragile,maheri2026zkapex}. That evidence concerns \emph{which
computation ran}, for a verifier that distrusts the operator. Ours is the
complement: the verifier is internal, the computation known, and the question
is whether the artifact about to be released meets a contract stated in exact
arithmetic. The two compose, since proving that the prescribed update ran says
nothing about the rounding it accumulated.

That numerical question has its own literature. Finite precision is known to
break guarantees derived in exact arithmetic, in differential privacy,
robustness certification, and verified training
\citep{mironov2012,voracek2023,szasz2025,jin2024round}, and verified numerics
supplies the countermeasure: rigorous enclosures for linear-system residuals,
approximate inverses, and low-rank updates
\citep{krawczyk1969,oishi2002,rump2010,hager1989}. These tools certify a
solution against a \emph{fixed} system. Cholesky downdates and iterative
refinement provide classical ways to update a numerical solve
\citep{higham2002,rump2010}; maintaining such a factor does not by itself
verify its relation to the changing exact system, so factor validation and
residual enclosures remain part of a release certificate. Sequential deletion removes that
fixture, since the retain-set system moves with every request while the
service advances a separately rounded state, so our ridge regression
realization maintains a certified relation between the two and converts it at
release time into a certificate against the current reference. ScaleGUN
studies a related transfer for graph propagation~\citep{yi2025scalegun};
\RRV{} instead transfers from finite-precision service state to the retain-set
system.

\section{From Mathematical to Executable Certificates}
\label{sec:problem}
\label{sec:release-contract}

An unlearning generator processes a declared set or sequence of deletion
requests and produces a finite-precision candidate \(\widetilde\theta\). The
deployment fixes an admissible set \(\mathcal A\) independently of that
candidate.
\ExecCert{} returns \textsc{Certified} only after establishing
\(\widetilde\theta\in\mathcal A\); otherwise the policy gate may repair,
recompute, re-encode, or withhold the candidate.

For evaluation, a candidate is \emph{conforming} when it belongs to
\(\mathcal A\), and a \emph{false release} is a nonconforming candidate
returned as \textsc{Certified}. An \textsc{Inconclusive} candidate may still
conform; a \emph{released artifact} is one that passes the policy gate.

\textbf{Notation.}
We write \(\theta^{\mathrm R}\) for the exact-real retraining reference after
the declared deletions and \(\widetilde\theta\) for the executed candidate.
Hats denote computed verifier quantities and circles exact-real evaluations of
computed intermediates; \cref{app:fp-conventions} gives the remaining
floating-point notation.

\subsection{Two \ExecCert{} Paths}
\label{sec:certificate-paths}
These paths correspond to two release requirements. Some deployments need to
verify a method's own certificate after finite-precision execution; others
specify fidelity to retraining on the current retain set, which a native
certificate may not express. Native closure addresses the first requirement,
whereas \RRV{} addresses the second.

\paragraph{Native-certificate closure.}
This path is appropriate when an unlearning method already specifies a
certificate of the form
\(q(\theta;\mathcal I)\leq b\), where \(q\) is the source certificate
quantity, \(b\) its prescribed threshold, and \(\mathcal I\) any committed
inputs or auxiliary state required by the source method. This defines the
admissible set
\(\mathcal A^{\mathrm{native}}=\{\theta\colon q(\theta;\mathcal I)\leq b\}\).
\ExecCert{} does not change this condition: it constructs a rigorous upper
bound \(q(\widetilde\theta;\mathcal I)\leq q^+\) for the executed candidate and
returns \textsc{Certified} whenever \(q^+\leq b\). Native closure therefore certifies the concrete
finite-precision artifact under the source method's own certificate.

\paragraph{Retraining-reference certification.}
This path is needed when a method's native certificate does not quantify how
close its executed output is to retraining on the current retain set. When a
deployment requires this fidelity directly, \ExecCert{} uses
\emph{\RRVFull{}} (\RRV{}).
\RRV{} defines a uniquely specified exact-real target
\(\theta^{\mathrm R}\) from the declared training objective, mutable model
component, and contract inputs. For a predeclared distance \(d\) and tolerance
\(\tau>0\), the release contract requires
\(d(\widetilde\theta,\theta^{\mathrm R})\leq\tau\).
Rather than computing \(\theta^{\mathrm R}\) exactly, \RRV{} certifies this
contract by constructing a rigorous upper bound on the distance.

For the Euclidean parameter contract, this takes the form
\begin{equation}
    \norm{\widetilde\theta-\theta^{\mathrm R}}_F
    \leq C.
    \label{eq:absolute-interface-goal}
\end{equation}
How \(C\) is obtained depends on the model class. For ridge regression heads
with symmetric positive-definite (SPD) normal matrices,
\OnDemandRRV{} constructs the required proof information for the candidate,
whereas \IncrementalRRV{} maintains it while deletion requests are processed.
The latter updates certified numerical state through rank-one transitions. For
differentiable strongly convex heads, \OnDemandRRV{} instead uses a post-update
gradient certificate, covering \(L_2\)-regularized logistic and softmax heads.

\subsection{Contracts Derived from the \RRV{} Bound}
\label{sec:downstream-contracts}

The absolute bound \(C\) can be translated into several release contracts.
For relative parameter fidelity, denote the predeclared tolerance by
\(\tau_{\rm par}\). The reverse triangle inequality gives
\begin{equation}
    \frac{\norm{\widetilde\theta-\theta^{\mathrm R}}_F}
         {\norm{\theta^{\mathrm R}}_F}
    \leq
    \frac{C}{\norm{\widetilde\theta}_F-C},
    \label{eq:relative-parameter-contract}
\end{equation}
whenever \(\norm{\widetilde\theta}_F>C\), so the contract is certified when the
right-hand side is at most \(\tau_{\rm par}\).

A deployment may instead care about outputs on a fixed set of input
representations \(Z\in\R^{m\times d}\), declared independently of the
candidate. For a parameter matrix with \(c\) output columns, it may require
\begin{equation}
    \frac{\norm{Z(\widetilde\theta-\theta^{\mathrm R})}_F}
         {\sqrt{mc}}
    \leq \tau,
    \label{eq:fixed-readout-contract}
\end{equation}
for which the parameter certificate immediately gives
\(\norm{Z(\widetilde\theta-\theta^{\mathrm R})}_F\leq\norm{Z}_2 C\), while
\cref{sec:rrv} develops a tighter error-centered certificate that preserves
the error direction, and the margin conditions in
\cref{cor:decision-agreement} further imply per-row decision agreement.
Contracts are checked separately; the same absolute bound also supports the
randomized-release and output-format interfaces in
\cref{sec:randomized-interfaces,sec:formats}.

\section{\RRVFull{}}
\label{sec:joint-certificate}
\label{sec:rrv}
Native-certificate closure is method-specific once the source certificate
quantity is fixed; its concrete instantiations appear in
\cref{sec:published-certificate-results}. The main technical development is
therefore \RRV{}, which must certify an executed candidate against a
retraining reference that is not explicitly computed.

\RRV{} can construct certificate evidence in two ways. \OnDemandRRV{}
constructs the required evidence when certification is requested, whereas
\IncrementalRRV{} reuses numerical state maintained across deletion requests.
The latter introduces an additional obligation: the
maintained finite-precision state must itself be certified against the exact
current retain-set system before it can support a release certificate.
Our ridge regression realization supports both paths through residual and
inverse-defect bounds; for differentiable strongly convex heads,
\OnDemandRRV{} uses a post-update gradient certificate.

The interface in \cref{sec:release-contract} concerns a single candidate after
a declared deletion sequence. From this point onward, we index internal service
state by deletion step \(t\) only to describe how incremental evidence evolves.
The same machinery also accommodates long-running services that expose
intermediate snapshots, in which case a declared subset of internal states may
serve as release checkpoints.

\subsection{Release-Time \RRV{} Certificates}
\label{sec:common-certificate}

For a candidate considered for release, \RRV{} must bound its distance to the
exact retraining reference \(\theta^{\mathrm R}\) without explicitly computing
that reference. We give two constructions of this parameter-error bound.

\paragraph{Ridge regression forward certificate.}
Let the exact retraining reference satisfy \(H\theta^{\mathrm R}=B\) with
\(H\in\R^{d\times d}\) nonsingular, and let \(P\) approximate \(H^{-1}\). Since
\(P(B-H\widetilde\theta)=PH(\theta^{\mathrm R}-\widetilde\theta)\), the
corrected residual would equal the parameter error if \(P=H^{-1}\). With an
approximate inverse, the verifier instead controls both the inverse defect and
the corrected residual:
\begin{equation}
    \norm{I_d-PH}_2\leq\delta<1,
    \qquad
    \norm{P(B-H\widetilde\theta)}_F\leq\rho .
    \label{eq:delta-rho-contract}
\end{equation}

\begin{lemma}[Retraining-reference forward certificate]
\label{thm:joint}
Under \cref{eq:delta-rho-contract},
\[
\norm{\widetilde\theta-\theta^{\mathrm R}}_F
\leq
C
:=
\frac{\rho}{1-\delta}.
\]
\end{lemma}

\paragraph{Strongly-convex post-update certificate.}
For a differentiable strongly convex objective, \RRV{} can instead construct
the parameter certificate directly from the final candidate. If \(g^+\)
upper-bounds its gradient norm under the current retain-set objective and
\(\mu^-\) lower-bounds the objective's strong-convexity constant, then
\(C=g^+/\mu^-\) bounds the candidate's distance to the exact minimizer.
This post-update
construction requires no inverse proof state and applies, for example, to
\(L_2\)-regularized logistic and softmax heads. The formal statement and
executable construction appear in
\cref{prop:strongly-convex-post-update,app:published-output-applications}.

For fixed readouts, the ridge regression construction admits a tighter directional
refinement. For a predeclared set of input representations
\(Z\in\R^{m\times d}\), the parameter certificate gives
\(\norm{Z(\theta^{\mathrm R}-\widetilde\theta)}_F\leq\norm Z_2 C\),
which treats the direction of the parameter error as unknown. The corrected
residual also provides a computed center \(\widehat\Delta\) for that error,
allowing the remaining uncertainty to be bounded around an estimated
direction.

\begin{proposition}[Error-centered readout certificate]
\label{prop:fixed-readout-error-center}
Let \(\Delta=\theta^{\mathrm R}-\widetilde\theta\) and suppose
\(\norm{\Delta}_F\leq C\) and
\(\norm{\Delta-\widehat\Delta}_F\leq R_\Delta\). Then, for any fixed
\(Z\in\R^{m\times d}\),
\begin{equation}
 \norm{Z\Delta}_F
 \leq
 \min\!\bigl\{
   \norm Z_2 C,
   \norm{Z\widehat\Delta}_F+\norm Z_2 R_\Delta
 \bigr\}.
 \label{eq:fixed-readout-bound}
\end{equation}
\end{proposition}

We denote the resulting generic and centered bounds by
\(C_Z^{\rm ball}=\norm Z_2 C\) and
\(C_Z^{\rm ctr}=\norm{Z\widehat\Delta}_F+\norm Z_2 R_\Delta\).
The centered bound preserves the estimated error direction and applies the
worst-case bound only to its remaining uncertainty. Its executable
construction appears in \cref{app:fixed-readout-executable}. The same
certificate yields binary score-sign and multiclass top-1 agreement under the
margin conditions in
\cref{cor:decision-agreement}.

\subsection{Maintaining Reference-Transfer Evidence}
\label{sec:proof-state}
\label{sec:propagation}

When a deployment chooses to maintain certificate evidence incrementally,
the verifier must relate the service's evolving finite-precision state to the
exact current retain-set system. The release-time certificate in
\cref{sec:common-certificate} ultimately concerns the latter, while the
service updates a separate stored numerical state. \IncrementalRRV{} therefore
maintains certified evidence connecting these two systems throughout the
deletion sequence.

Let \(D_t=\{(x_i,y_i)\}_{i=1}^{n_t}\) be the retain set after \(t\) committed
deletions, with feature and target matrices \(X_t\) and \(Y_t\). For
\(\lambda>0\), define
\begin{equation}
    H_t=X_t^\top X_t+\lambda I_d,
    \qquad
    B_t=X_t^\top Y_t,
    \qquad
    H_t\theta_t^{\mathrm R}=B_t.
    \label{eq:ridge-system}
\end{equation}
Deleting \((x,y)\) changes the exact statistics by
\(H_{t+1}=H_t-xx^\top\) and \(B_{t+1}=B_t-xy^\top\).

The service instead stores finite-precision statistics
\(\widetilde H_t,\widetilde B_t\). \IncrementalRRV{} tracks their discrepancy
from the exact retain-set statistics through componentwise enclosures
\begin{equation}
    \abs{\widetilde H_t-H_t}\leq\mathcal U_{H,t},
    \qquad
    \abs{\widetilde B_t-B_t}\leq\mathcal U_{B,t}.
    \label{eq:stored-retained-budgets}
\end{equation}
These bounds certify how far the numerical state maintained by the service may
have drifted from the exact retain-set system.

The second maintained component is reusable inverse evidence. Rather than
reconstructing an inverse certificate after every deletion,
\IncrementalRRV{} carries an approximate inverse \(P_t\), a computed defect
tracker \(\widehat E_t\), and a certified radius \(r_t\) satisfying
\begin{equation}
    \norm{
        I_d-P_t\widetilde H_t-\widehat E_t
    }_F
    \leq r_t.
    \label{eq:defect-tracker-invariant}
\end{equation}
For each signed rank-one change, a guarded Sherman--Morrison transition updates
\(P_t\), the defect tracker, and the statistics enclosures. The propagated
radius accounts for the previous certified uncertainty together with
arithmetic error in the update and tracker evaluation. Every denominator,
matrix-vector product, and state transition is checked before commit; a failed
guard invokes the configured rebuild. A successful rank-one transition costs
\(O(d^2)\), compared with \(O(d^3)\) for dense inverse reconstruction.

At release, these two maintained components are combined to recover the
quantities required by the certificate in
\cref{sec:common-certificate}. Suppressing \(t\), let
\(E=I_d-P\widetilde H\). The exact retain-set residual and inverse defect obey
\begin{equation}
\begin{aligned}
    B-H\widetilde\theta
    &=
    (\widetilde B-\widetilde H\widetilde\theta)
    +(B-\widetilde B)
    -(H-\widetilde H)\widetilde\theta,\\
    I_d-PH
    &=
    E+P(\widetilde H-H).
\end{aligned}
\label{eq:reference-transfer-identity}
\end{equation}
The first identity transfers the stored-system residual to the exact
retain-set residual; the second transfers the stored-system inverse defect to
the exact retain-set system.

Here and below, an overline denotes a certified upper bound on the enclosed
norm, and \(\operatorname{up}(\cdot)\) denotes componentwise outward-rounded
evaluation.

Let
\(\widehat R=\fl(\widetilde B-\widetilde H\widetilde\theta)\),
let \(\mathcal U_R\) enclose residual-evaluation error together with the
stored-to-retain-set correction, and let
\(\varepsilon_{P\widehat R}\) enclose the arithmetic error in
\(P\widehat R\). Combining the maintained statistics and inverse evidence gives
\begin{equation}
\boxed{
\begin{aligned}
 \norm{I_d-PH}_2
 &\leq \delta_{\rm ret}
 :=\operatorname{up}\!\left(
   \delta_{\rm stored}
   +\overline{\norm P_2}\,
    \overline{\norm{\mathcal U_H}_2}\right),\\
 \norm{P(B-H\widetilde\theta)}_F
 &\leq \rho_{\rm ret}
 :=\operatorname{up}\!\left(
   \underbrace{\norm{\fl(P\widehat R)}_F}_{\text{corrected residual}}
   +\underbrace{\varepsilon_{P\widehat R}}_{\text{evaluation error}}
   +\underbrace{
      \norm{\operatorname{up}(\abs P\mathcal U_R)}_F
    }_{\text{reference-transfer error}}\right),\\
 C&=\frac{\rho_{\rm ret}}{1-\delta_{\rm ret}}.
\end{aligned}}
\label{eq:reference-transfer-certificate}
\end{equation}
Thus the evidence maintained during the deletion sequence is converted only
when certification is required into the exact retain-set quantities consumed
by the release-time certificate.

The componentwise outward-rounded construction appears in
\cref{app:joint_bounds}; the rank-one recurrence and propagation proofs appear
in
\cref{app:nonsymmetric-recurrence,app:statistics-error-invariant,%
app:sequential-inverse}.

\subsection{Transactional Release Conformance}
\label{sec:policy}

A release-time certificate may be constructed by either \OnDemandRRV{} or
\IncrementalRRV{} as in \cref{sec:proof-state}. In either case, the certificate
must apply to the exact artifact that becomes visible.

Let \(\mathcal R\) denote the declared release checkpoints. At
\(t\notin\mathcal R\), the service may commit verified internal state without
emitting an artifact. At \(t\in\mathcal R\), the current candidate is released
only after every declared contract certifies. Any repair, recomputation, or
re-encoding that changes the candidate produces a new artifact and therefore
requires recertification. A proof-state rebuild, by contrast, changes only the
verifier evidence and leaves the candidate itself unchanged.

\begin{theorem}[\IncrementalRRV{} release conformance]
\label{thm:policy}
Suppose \IncrementalRRV{} is initialized with valid statistics and
inverse-defect enclosures, accepts only guarded transitions that preserve both
invariants, and at each declared release checkpoint applies
\cref{eq:reference-transfer-certificate} and the declared contract checks to
the concrete candidate. If candidate-changing remediation is recorded as a
new candidate and recertified, and release is atomic with respect to the
certified artifact, then every released artifact satisfies every declared
retraining-reference contract.
\end{theorem}

The theorem states the concrete \IncrementalRRV{} initialization, transition,
reference-transfer, contract-checking, and transactional-release obligations;
\OnDemandRRV{} is the special case that constructs sound evidence only at a
release checkpoint. Proofs and precise transactional assumptions appear in
\cref{app:transactional-safety-proof}. Successful
rank-\(k\) incremental propagation costs \(O(kd^2+kdc)\), excluding output
generation; complete release and rebuild accounting appears in
\cref{app:transactional-safety-proof}.

\subsection{Derived Release Contracts}

\label{sec:randomized-interfaces}
\label{sec:gaussian-interface}
The same bound composes with Gaussian and Laplace release guarantees, including
executable closure of Newton direct-perturbation radii
~\citep{pandey2026,zou2025highdim}; the full interfaces appear in
\cref{app:randomized-release-interfaces}.

\paragraph{Format compatibility.}
\label{sec:formats}
A final representation may itself be incapable of satisfying the declared
contract. Let \(\widehat\theta\) be a certified reference center with
\(\norm{\widehat\theta-\theta^{\mathrm R}}_F\leq r_{\rm ref}\). For a finite
scalar format \(\mathcal F\subset\R\), let
\(\underline d_{\mathcal F}(\widehat\theta)\) be a certified lower bound on the
distance from \(\widehat\theta\) to \(\mathcal F^{d\times c}\), and let
\(\overline s_{\rm ref}\) be a certified upper bound on
\(\norm{\widehat\theta}_F\).

\begin{proposition}[Finite-format impossibility]
\label{prop:format-impossibility}
Assume \(\theta^{\mathrm R}\ne0\).  If
\begin{equation}
 \underline d_{\mathcal F}(\widehat\theta)-r_{\rm ref}
 >
 \tau_{\rm par}\left(
 \overline s_{\rm ref}+r_{\rm ref}
 \right),
 \label{eq:format-impossibility}
\end{equation}
then no \(\theta_{\mathcal F}\in\mathcal F^{d\times c}\) satisfies the relative
contract in \cref{eq:relative-parameter-contract}.
\end{proposition}

The format is then infeasible for the declared contract; the proof and
codebook construction appear in
\cref{app:format-impossibility-proof,app:format-reachability}.

\section{Experiments}
\label{sec:experiments}
\textbf{Evaluation questions.}
We ask whether reference transfer prevents false releases without making
certificates overly conservative (\cref{sec:stateful-rrv-results}), whether
incremental evidence reduces lifecycle cost without weakening certification
(\cref{sec:guarantee-matched-results}), and how \ExecCert{} affects four
published methods without changing their update rules
(\cref{sec:published-certificate-results}). We then test when producer
precision or output representation limits the contract
(\cref{sec:llm-format-remediation}).

The sequential service certifies a \(d=3073\) binary ridge head on frozen
Llama-3.2-3B-Instruct~\citep{meta2024llama32} representations. Controlled
diagnostics use DINOv2~\citep{oquab2024dinov2learningrobustvisual} on
CIFAR-10~\citep{CIFAR_10} and RoBERTa~\citep{liu2019robertarobustlyoptimizedbert}
on AG-News~\citep{AGNEWS}; the published applications cover linear, logistic,
and graph methods. Comparisons hold the candidate and contract fixed.
Independent retraining and high-precision intervals supply offline labels.
Protocols and provenance appear in
\cref{app:published-method-applications,app:service-evaluation,%
app:reference-reproducibility}.

\subsection{Reference-Transfer Soundness and Tightness}
\label{sec:stateful-rrv-results}
\label{sec:release-policy-results}

Stored-equation verification admits nearly all nonconforming candidates on the
matched stress population, whereas \RRV{} produces no false release
(\cref{tab:hard-stream-conformance}). On the FP32 stress population,
\IncrementalRRV{} initially certifies 4,838 of 7,293 conforming outputs without
a false release; proof-only rebuilds raise conditional coverage to 67.24\%,
compared with 99.97\% for fresh-factor verification
(\cref{sec:llm-format-remediation}). On a separate fixed population, \RRV{}
encloses every realized error across the tolerance sweep, and its certificate
remains within about 1\% of the independent reference for 95\% of outputs
(\cref{fig:release-contract-frontier}).

\begin{figure}[tbp]
  \centering
  \includegraphics[width=\textwidth]{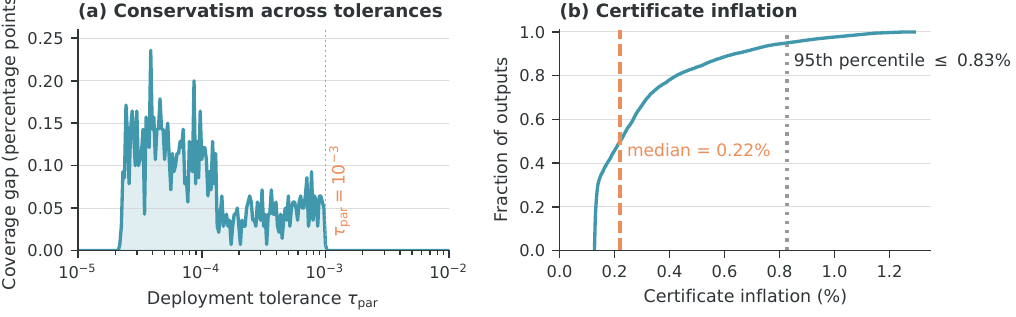}
  \caption{\textbf{\RRV{} certificates closely track realized error.}
  Incremental componentwise \RRV{} and independent retained-row reconstruction
  evaluate the same fixed FP32 deletion outputs at every tolerance. Coverage
  gap is the conforming fraction minus the certified
  fraction; certificate inflation is
  \(\tau_{\rm cert}/\tau_{\rm true}-1\). No rebuild or producer remediation
  changes the population.}
  \label{fig:release-contract-frontier}
\end{figure}

\subsection{Evidence Reuse in Sequential Services}
\label{sec:guarantee-matched-results}

We compare fresh and maintained verified factors with \IncrementalRRV{} on the
same candidate outputs, fixed readout, and release contract. A direction-free
incremental variant tests whether reuse alone is sufficient.

\begin{table}[tbp]
  \caption{\textbf{Verification of the same frozen 3B safety-head outputs.}
  Latency summarizes 10 repetitions at each of 21 states.
  No method scans retained rows; state excludes those rows.}
  \label{tab:guarantee-matched-service}
  \centering\footnotesize
  \setlength{\tabcolsep}{1.8pt}
  \renewcommand{\arraystretch}{1.06}
  \begin{tabular*}{\columnwidth}{@{\extracolsep{\fill}}lcrr@{}}
    \toprule
    Method & \shortstack{Certified\\release states} & Med./P95 (s) & State (MiB) \\
    \midrule
    \OnDemandRRV{} (fresh factor) & 21/21 & 2.611/3.593 & 286.5 \\
    Maintained verified factor & 21/21 & 2.453/3.678 & 358.6 \\
    \IncrementalRRV{} (norm-ball) & 0/21 & 0.801/0.859 & 360.3 \\
    \IncrementalRRV{} (directional) & 21/21 & 1.056/1.163 & 430.6 \\
    \bottomrule
  \end{tabular*}
\end{table}

The maintained factor removes the weak-baseline explanation: it matches the
coverage of fresh factor, while directional \IncrementalRRV{} uses roughly 40\% of
fresh-factor latency at matched coverage
(\cref{tab:guarantee-matched-service}). The norm-ball variant fails on the same
states, so reuse preserves strength only when the proof state retains
contract-relevant error direction. This ordering holds throughout the nonempty
tolerance sweep. Replacing the readout-norm upper bound by a certified lower
bound still leaves every norm-ball certificate above tolerance, ruling out
readout-norm looseness as the explanation. A separate decision audit covers
about 95\% of declared safety-head decisions; observed agreement is not used
as certified coverage.

Measured-component lifecycle accounting combines complete stream replays with
the paired release-check timings. Fresh factor is cheapest for sparse checks;
directional \IncrementalRRV{} crosses it at about 22 checks per 1,000 requests
and remains cheaper than maintained factor throughout the evaluated schedules
(\cref{fig:deployment-tradeoffs}\textup{(a)}). Under every-request checking, it
reduces lifecycle cost by 42.7\% relative to maintained factor. The crossover
is workload-specific because it trades continuous proof maintenance against
release-time construction.

\subsection{Published-Method Applications}
\label{sec:published-certificate-results}

We attach \ExecCert{} to four existing unlearning methods: Newton-update
removal~\citep{guo2020}, Newton direct perturbation~\citep{zou2025highdim},
ScaleGUN~\citep{yi2025scalegun}, and CEU~\citep{wu2023ceu}. In each case, the
original update rule and emitted candidate are left unchanged; \ExecCert{} is
applied only at release time. The check closes either the method's native
certificate quantity or a predeclared retraining-reference contract.

\begin{table}[tbp]
  \caption{\textbf{Executable certification on published unlearning methods.}
  \ExecCert{} is applied to the emitted candidate without changing the
  underlying update rule.}
  \label{tab:published-method-applications}
  \centering
  \small
  \setlength{\tabcolsep}{4pt}
  \renewcommand{\arraystretch}{1.12}
  \begin{tabularx}{\linewidth}{@{}>{\raggedright\arraybackslash}p{0.25\linewidth}
      >{\raggedright\arraybackslash}X@{}}
    \toprule
    Method & Release-time outcome \\
    \midrule
    Newton removal
      & Native closure shifts the retraining boundary by
        \textbf{0--4 requests}; direct recertification also identifies
        conservative source stops. \\
    Newton direct perturbation
      & At FP32, the median direct bound is \textbf{2.28$\times$ tighter}
        than composition; paired FP64 results retain an absolute gap. \\
    ScaleGUN
      & All \textbf{15/15} cheap-update certificates
        remain valid. \\
    CEU
      & Median realized error is \textbf{8.9\%}; the median certified upper
        bound is \textbf{14.1\%}. \\
    \bottomrule
  \end{tabularx}
\end{table}

Native closure can move Newton-removal retraining earlier, while direct
recertification identifies conservative source stops. For Newton direct
perturbation, direct certification remains tighter than direction-free
composition under both execution precisions; we report an absolute FP64 gap
because the direct bound can approach zero. \ExecCert{} preserves the
ScaleGUN cheap-update certificates. For CEU, the declared tolerance is not
supported, but post-update \RRV{} certifies the weaker fidelity of the emitted
heads. These cases show that release-time certification can change a decision,
tighten or preserve a certificate, or identify the fidelity supported by an
output. Source instrumentation and supplementary results appear in
\cref{app:published-method-applications,app:theorem-input-audits}.

\begin{figure}[tbp]
  \centering
  \begin{minipage}[t]{0.485\linewidth}
    \centering
    \textbf{(a) Evidence reuse}

    \vspace{2pt}
    \includegraphics[width=\linewidth]{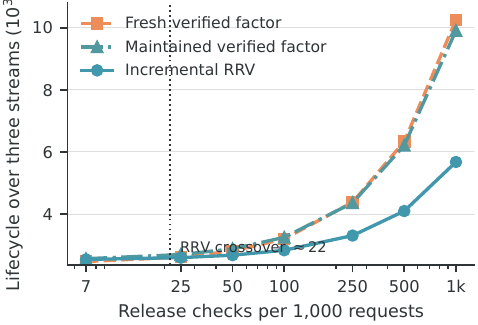}
  \end{minipage}
  \hfill
  \begin{minipage}[t]{0.485\linewidth}
    \centering
    \textbf{(b) Format feasibility}

    \vspace{2pt}
    \includegraphics[width=\linewidth]{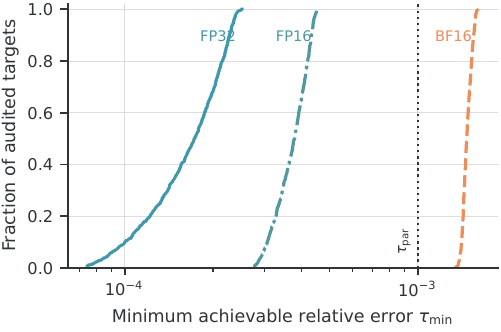}
  \end{minipage}
  \caption{\textbf{Deployment choices affect certification cost and
  feasibility.} (a) Full-stream maintenance measurements combined with matched
  release-check timings produce the lifecycle curves. The maintained-factor
  crossover with fresh construction lies beyond the plotted low-frequency
  region, at about 214 checks per 1,000 requests.
  (b) The horizontal axis is the minimum achievable relative error
  \(\tau_{\min}\). A format is ruled out when its certified lower bound exceeds
  the declared contract \(\tau_{\rm par}\); BF16 is unreachable, while
  equal-size FP16 remains reachable.}
  \label{fig:deployment-tradeoffs}
\end{figure}

\subsection{Precision and Representation Boundaries}
\label{sec:llm-format-remediation}

The preceding experiments vary certificate construction for a fixed candidate.
We finally ask when the candidate or its representation, rather than the
verifier, limits the contract.

\textbf{Producer precision.}
With FP32 statistics, fresh-factor \OnDemandRRV{} covers materially more
conforming candidates than \IncrementalRRV{} after proof-only rebuilds; both
paths attain complete coverage with the FP64 producer. Changing precision also
changes the emitted candidate, so this is a producer-policy comparison rather
than a certificate-only comparison. Full timing and state results appear in
\cref{tab:precision-boundary,app:sequential-proof-state-behavior}.

\textbf{Release format.}
BF16 cannot represent any artifact satisfying the declared tolerance, whereas
FP16 and FP32 remain feasible (\cref{fig:deployment-tradeoffs}\textup{(b)}).
Thus, some failures require changing the producer or encoding rather than
tightening verification.

\section{Discussion and Conclusion}
\label{sec:discussion}
\label{sec:conclusion}

Our results separate three release-time failure modes: a mathematical
certificate may not close for the executed artifact, a sound verifier may be
too weak or costly for the release schedule, or the producer and output format
may preclude the contract. \ExecCert{} identifies whether remediation should
target verification or the released artifact.

Certification strategy depends on the contract and release schedule. Native
closure uses the quantity governing a source guarantee; \RRV{} targets fidelity
to current retain-set retraining. On-demand construction favors sparse checks,
while incremental proof state amortizes frequent certification. \ExecCert{}
supplements the upstream guarantee: parameter or fixed-readout fidelity does
not establish semantic forgetting or multi-round adaptive privacy. Our
incremental implementation covers frozen representations with
moderate-dimensional mutable ridge heads; \cref{app:scope} gives the remaining
conditions.

\subsection*{AI Use Statement}
We used generative AI tools to assist with code implementation and refactoring,
figure and \LaTeX{} editing, and language polishing. The authors reviewed and
verified all mathematical statements, experimental designs, interpretations of
source methods, code changes, results, and manuscript text. We take
responsibility for the final content of this work, including text, claims, and
artifacts produced with the aid of generative AI.

\subsection*{Reproducibility Statement}
The appendix provides complete proofs, executable-verification details,
published-source provenance, experimental protocols, and independent reference
construction. The accompanying artifact contains the implementation,
fixed configurations, source instrumentation patches, and canonical records
used for the reported tables and figures.

\bibliography{references}
\bibliographystyle{iclr2027_conference}

\clearpage
\appendix
\tcolorboxenvironment{theorem}{
  enhanced, breakable, colback=blue!2, colframe=blue!35!gray,
  boxrule=0.55pt, arc=1mm, left=1mm, right=1mm,
  top=0.8mm, bottom=0.8mm, before skip=5pt, after skip=5pt
}
\tcolorboxenvironment{proposition}{
  enhanced, breakable, colback=blue!2, colframe=blue!35!gray,
  boxrule=0.55pt, arc=1mm, left=1mm, right=1mm,
  top=0.8mm, bottom=0.8mm, before skip=5pt, after skip=5pt
}
\tcolorboxenvironment{lemma}{
  enhanced, breakable, colback=blue!2, colframe=blue!35!gray,
  boxrule=0.55pt, arc=1mm, left=1mm, right=1mm,
  top=0.8mm, bottom=0.8mm, before skip=5pt, after skip=5pt
}
\tcolorboxenvironment{corollary}{
  enhanced, breakable, colback=blue!2, colframe=blue!35!gray,
  boxrule=0.55pt, arc=1mm, left=1mm, right=1mm,
  top=0.8mm, bottom=0.8mm, before skip=5pt, after skip=5pt
}
\section{Scope and Limitations}
\label{app:scope}

Native closure inherits the scope of the source method's own certificate;
\RRV{} instead certifies only the declared distance to its uniquely specified
retraining reference.

Fixed-readout claims apply to the input representations declared before the
candidate is inspected, and say nothing about inputs outside that set.
Randomized-release calibration assumes a predeclared common scale; adaptive
release policies, in which the noise scale reacts to the candidate, require
separate analysis.

Our reusable proof state currently targets moderate-dimensional mutable heads.
Extending incremental certification to larger trainable blocks will require
exploiting additional structure, since dense inverse evidence does not scale
to that regime. Unlike statistical deletion capacity, proof-state lifetime is
instance-dependent and governed by conditioning rather than a universal
deletion count: a failed guard triggers certified rebuilding, which restores
the invariants at higher cost but imposes no fixed bound on the number of
deletions a service can absorb.

\section{Proofs of Main Results}
\label{app:main-proofs}

\subsection{Release-Time \RRV{} Certificates}

We first establish the parameter-error bound and its error-centered
fixed-readout refinement. The first is the classical residual--inverse-defect
argument from verified linear algebra~\citep{krawczyk1969,oishi2002,rump2010};
we retain its short derivation because it is the main certificate interface.

\begin{proof}[Proof of \cref{thm:joint}]
Let \(\Delta=\theta^{\mathrm R}-\widetilde\theta\). Since
\(H\theta^{\mathrm R}=B\),
\[
\begin{aligned}
\Delta&=(I-PH)\Delta+P(B-H\widetilde\theta),\\
\norm{\Delta}_F
&\leq
\norm{I-PH}_2\norm{\Delta}_F
+\norm{P(B-H\widetilde\theta)}_F
\leq
\delta\norm{\Delta}_F+\rho.
\end{aligned}
\]
Thus \(\norm{\Delta}_F\leq\rho/(1-\delta)=C\). The reverse triangle
inequality gives
\(\norm{\theta^{\mathrm R}}_F\geq\norm{\widetilde\theta}_F-C\), so whenever
\(\norm{\widetilde\theta}_F>C\),
\begin{equation}
\frac{\norm{\widetilde\theta-\theta^{\mathrm R}}_F}
     {\norm{\theta^{\mathrm R}}_F}
\leq
\frac{C}{\norm{\widetilde\theta}_F-C}.
\label{eq:relative-certificate}
\end{equation}
\end{proof}

\begin{proof}[Proof of \cref{prop:fixed-readout-error-center}]
The first bound follows from
\(\norm{Z\Delta}_F\leq\norm{Z}_2\norm{\Delta}_F
\leq\norm{Z}_2C\). For the second, write
\(\Delta=\widehat\Delta+(\Delta-\widehat\Delta)\) and apply the triangle
inequality together with the stated center enclosure. Taking the smaller of
the two valid bounds gives \cref{eq:fixed-readout-bound}.
\end{proof}

The absolute parameter bound also certifies decisions with sufficient margin.
\begin{corollary}[Certified linear-decision agreement]
\label{cor:decision-agreement}
Fix a representation \(z\in\R^d\). For a binary head satisfying
\(\norm{\widetilde\theta-\theta^{\mathrm R}}_2\leq C\), the executed and
retraining-reference signs agree when
\(\abs{z^\top\widetilde\theta}>\norm z_2C\). For a multiclass head, let
\(\norm{\widetilde\Theta-\Theta^{\mathrm R}}_F\leq C\) and
\(a=\operatorname*{arg\,max}_j z^\top\widetilde\theta_j\). It remains the
unique top class
under \(\Theta^{\mathrm R}\) when
\[
z^\top\widetilde\theta_a-z^\top\widetilde\theta_j
>\sqrt{2}\,\norm z_2C
\qquad\text{for every }j\ne a.
\]
\end{corollary}
\begin{proofbox}[Proof]
Cauchy--Schwarz bounds the binary score change by \(\norm z_2C\).
For the multiclass case, writing
\(\Delta_j=\theta_j^{\mathrm R}-\widetilde\theta_j\) gives
\[
\abs{z^\top(\Delta_a-\Delta_j)}
\leq\norm z_2\norm{\Delta_a-\Delta_j}_2
\leq\sqrt{2}\,\norm z_2 C.
\]
The strict margins preserve the stated decisions.
\end{proofbox}

For differentiable strongly convex objectives, a post-update gradient
enclosure provides a separate construction of the parameter-error bound.

\begin{proposition}[Strongly-convex post-update certificate]
\label{prop:strongly-convex-post-update}
Let \(L:\R^d\to\R\) be differentiable and \(\mu\)-strongly convex, with
minimizer \(\theta^{\mathrm R}\). Suppose the verifier constructs
\[
 g^+\geq\norm{\nabla L(\widetilde\theta)}_2,
 \qquad 0<\mu^-\leq \mu.
\]
Then
\[
 \norm{\widetilde\theta-\theta^{\mathrm R}}_2
 \leq \frac{g^+}{\mu^-}=:C.
\]
\end{proposition}

The absolute bound is the standard gradient--distance consequence of strong
convexity~\citep{nesterov2004}. A relative contract follows by applying
\cref{eq:relative-certificate}.

\subsection{Randomized Release Interfaces}
\label{app:gaussian-release-interface}
\label{app:randomized-release-interfaces}

These fixed-state interfaces use a predeclared common noise scale. If \RRV{}
certifies
\(\norm{\widetilde\theta-\theta^{\mathrm R}}_2\leq C\), the
equal-covariance Gaussian testing identity of \citet{pandey2026} gives
\[
\varepsilon_*=
\frac{\norm{\widetilde\theta-\theta^{\mathrm R}}_2}{\sigma}
\leq\frac{C}{\sigma}.
\]
For isotropic radial Laplace releases with common predeclared scale \(s\), we
use the density \(q_\mu(z)\propto\exp(-\norm{z-\mu}_2/s)\). Sampling draws a
uniform direction on the Euclidean unit sphere and an independent radius from
\(\operatorname{Gamma}(d,s)\); the same rule and scale apply to both centers.
The triangle inequality gives,
for every measurable event \(A\),
\[
Q_{\rm exec}(A)\leq e^{C/s}Q_{\rm ref}(A),
\qquad
Q_{\rm ref}(A)\leq e^{C/s}Q_{\rm exec}(A).
\]
Thus \(\sigma\geq C/\varepsilon\) gives the corresponding Gaussian trade-off
guarantee, while \(s\geq C/\varepsilon\) gives a two-sided Laplace comparison.
Here \(\varepsilon_*\) is the equal-covariance Gaussian testing parameter; it
is not by itself a conventional \((\varepsilon,\delta)\)-DP guarantee.
If an upstream analysis supplies radius \(r_{\rm nom}\) and executable closure
adds \(C_{\rm exec}\), the Laplace result applies with
\(C=r_{\rm nom}+C_{\rm exec}\).

\subsection{Nonsymmetric Defect Recurrence}
\label{app:nonsymmetric-recurrence}

All quantities in this subsection are evaluated in exact real arithmetic.
Let \(H_s,P\in\R^{d\times d}\), \(x\in\R^d\), and
\(s\in\{-1,+1\}\). Define
\[
\begin{aligned}
 E&=I_d-PH_s, & q&=Px,\\
 z&=P^\top x, & \alpha&=1+s x^\top q.
\end{aligned}
\]
For \(\alpha\neq0\), let
\[
 H_s^+=H_s+sxx^\top,
 \qquad
 P^+=P-\frac{s}{\alpha}qz^\top.
\]
The update for \(P^+\) is the Sherman--Morrison formula~\citep{hager1989}.
The following identity propagates the stored-system inverse defect without
assuming that \(P\) is symmetric. \Cref{app:sequential-inverse} constructs its
guarded floating-point realization.

\begin{lemma}[Exact nonsymmetric defect recurrence]
\label{lem:defect-recurrence}
For arbitrary \(H_s,P\in\R^{d\times d}\) with \(\alpha\neq0\), including
nonsymmetric \(P\), the post-update defect
\(E^+=I_d-P^+H_s^+\) satisfies
\[
 E^+
 =
 E-q\frac{s}{\alpha}
 \left(x^\top-z^\top H_s\right).
\]
The recurrence requires \(O(d^2)\) arithmetic and \(O(d^2)\) state.
\end{lemma}

\begin{proof}[Proof of \cref{lem:defect-recurrence}]
Because \(z^\top=x^\top P\), expansion of
\(E^+=I_d-P^+H_s^+\) gives
\[
 E^+=E-sqx^\top+\frac{s}{\alpha}qz^\top H_s
      +\frac{1}{\alpha}q(z^\top x)x^\top .
\]
Let \(\beta=z^\top x=x^\top q\).  Since \(\alpha=1+s\beta\), the
coefficient of \(qx^\top\) equals
\[
 -s+\frac{\beta}{\alpha}
 =-\frac{s}{\alpha}.
\]
Substitution yields
\(E^+=E-(s/\alpha)q(x^\top-z^\top H_s)\).
\end{proof}

\subsection{Transactional Conformance and Cost}
\label{app:transactional-safety-proof}

The transactional theorem applies to any sound verification backend. Its
incremental corollary combines the maintained statistics and inverse-defect
invariants with the release-time contract checks.

Here guarded committed transitions require every numerical denominator and
interval to pass its guard, every certificate evaluation to remain within the
certified arithmetic scope, and every rank-one subupdate and invariant check
to complete before commit. At a release checkpoint, the service exposes an
artifact only after every declared contract certifies.

\paragraph{Transactional justification.}
The initialization hypotheses give the two proof-state invariants at the first
state. \Cref{thm:statistics-error-bound,thm:radius} preserve them under every
accepted guarded transition. At a release checkpoint, the transfer and release
constructions in \cref{app:joint_bounds,app:release-certificates} therefore
establish every declared contract for the concrete candidate presented to the
policy gate. If remediation changes that candidate, the original certificate
no longer authorizes release and the new candidate must be certified
separately. Atomic release ensures that the visible artifact is exactly the
candidate covered by the successful certificates.

\paragraph{Cost accounting.}
Let \(G_t\) denote the cost of advancing the selected output generator at
request \(t\), let \(k_t\) be the number of rank-one subupdates for that
request, and set \(m=0\) when no fixed readout is declared. Each request adds
\(O(k_td^2+k_tdc)\) proof-state work; a release adds
\(O(d^2c+mdc)\) certificate work. Rebuilding proof state from certified
\((\widetilde H_t,\widetilde B_t)\) costs \(O(d^3+d^2c)\). Reconstructing those
statistics from retained rows adds \(O(n_td^2+n_tdc)\).
Let \(M_{\rm proof}\) count dense proof-state rebuilds that reuse certified
sufficient statistics, and let \(\mathcal R_{\rm data}\) be the index set of
requests whose rebuilds reconstruct those statistics from \(n_t\) retained
rows. Over \(N\) requests,
the total cost is
\[
 \begin{aligned}
 O\!\biggl(&\sum_{t=1}^N G_t
 +\sum_{t=1}^N(k_td^2+k_tdc)\\
 &+\abs{\mathcal R}(d^2c+mdc)
 +M_{\rm proof}(d^3+d^2c)\\
 &+\sum_{t\in\mathcal R_{\rm data}}
 (n_td^2+n_tdc+d^3+d^2c)\biggr).
 \end{aligned}
\]
The every-request-release policy is the special case
\(\abs{\mathcal R}=N\).

\subsection{Finite-Format Impossibility}
\label{app:format-impossibility-proof}

Let \(\theta^{\mathrm R}\in\R^{d\times c}\) be the exact retraining target,
and let \(\widehat\theta\) satisfy
\(\norm{\widehat\theta-\theta^{\mathrm R}}_F\leq r_{\rm ref}\). For a finite
scalar format \(\mathcal F\subset\R\), let
\(\underline d_{\mathcal F}(\widehat\theta)\) be a certified lower bound on
the distance from \(\widehat\theta\) to \(\mathcal F^{d\times c}\), and let
\(\overline s_{\rm ref}\) be a certified upper bound on
\(\norm{\widehat\theta}_F\).

\begin{proof}[Proof of \cref{prop:format-impossibility}]
For any representable
\(\theta_{\mathcal F}\in\mathcal F^{d\times c}\), the reverse triangle
inequality and the certified reference enclosure give
\[
\begin{aligned}
\norm{\theta_{\mathcal F}-\theta^{\mathrm R}}_F
&\geq \norm{\theta_{\mathcal F}-\widehat\theta}_F
      -\norm{\widehat\theta-\theta^{\mathrm R}}_F\\
&\geq
\underline d_{\mathcal F}(\widehat\theta)-r_{\rm ref}.
\end{aligned}
\]
The ordinary triangle inequality gives
\[
\norm{\theta^{\mathrm R}}_F
\leq \overline s_{\rm ref}+r_{\rm ref}.
\]
Under \cref{eq:format-impossibility}, the lower bound on the numerator is
strictly larger than \(\tau_{\rm par}\) times this upper bound on the denominator.
Thus every representable tensor violates the declared relative contract.
\end{proof}

\section{Soundness of \IncrementalRRV{}}
\label{app:domains_proofs}

This appendix gives the executable verifier specification needed to establish
the soundness of the incremental ridge regression backend of \RRV{}.
\Cref{alg:rrv-full} gives its transactional release protocol.
\begin{algorithm}[H]
  \caption{\IncrementalRRV{} release protocol}
  \label{alg:rrv-full}
  \footnotesize
  \begin{algorithmic}[1]
    \Require Committed service and proof state, deletion request \(t\), release
      schedule \(\mathcal R\), output generator, and one or more declared
      contracts: parameter fidelity, fixed-readout fidelity, or both
    \State Provisionally update the statistics, their retain-set enclosure,
      and the inverse-defect proof state.
    \If{a guard or proof-state invariant check fails}
      \State Apply the configured proof-state rebuild.
      \If{the rebuild does not reestablish the required invariant}
        \State Discard the transition; \Return internal \textsc{State-Fail}.
      \EndIf
    \EndIf
    \State Advance the output-generator state for request \(t\).
    \If{\(t\notin\mathcal R\)}
      \State Atomically commit the verified internal state; emit no output.
      \State \Return \textsc{NoRelease}
    \EndIf
    \State Materialize the current executed candidate artifact
      \(\widetilde\theta_t\).
    \State Construct \((\delta,\rho,C)\) and every declared contract certificate.
    \If{some declared contract is inconclusive}
      \State Withhold the candidate and apply the configured remediation.
      \State Reconstruct proof state as needed and recertify any changed candidate.
    \EndIf
    \If{every declared contract now certifies}
      \State Atomically commit state and certificates; release the candidate.
      \State \Return \textsc{Certified}
    \Else
      \State Atomically commit verified internal state; withhold the candidate.
      \State \Return \textsc{Inconclusive}
    \EndIf
  \end{algorithmic}
\end{algorithm}

The verifier maintains statistics and inverse-defect enclosures, transfers
them to the exact retain-set system, and constructs the release certificate in
\cref{app:statistics-error-invariant,app:sequential-inverse,app:joint_bounds,app:release-certificates}.

\subsection{Floating-point conventions and notation map}
\label{app:fp-conventions}

\paragraph{Notation and exact-real interpretation.}
We use the notation and shapes of \cref{sec:proof-state}; array inequalities
and absolute values remain componentwise. A hat denotes a stored computed
value, and a circle superscript denotes exact-real evaluation of computed
intermediates. The symbol \(\mathcal U\) denotes a componentwise enclosure,
\(\varepsilon\) a local arithmetic bound, and \((\delta,\rho)\) the hypotheses
consumed by \cref{thm:joint}. Stored entries are interpreted as their exact
real values.

\paragraph{Standard arithmetic facts.}
We use the componentwise floating-point models for basic arithmetic, dot
products, matrix--vector products, and matrix products from
\citet[Ch.~3]{higham2002}. Exact residuals returned by \textsc{TwoSum} use the
error-free transformation of \citet{ogita2005}. We invoke these standard facts
without rederiving them below; every displayed endpoint is evaluated outward,
and a failed arithmetic-scope guard returns \textsc{Inconclusive}.

\paragraph{Storage and verifier precision.}
The stored sufficient statistics may use
\(\mathbb F_{\rm s}\in\{\mathbb F_{32},\mathbb F_{64}\}\), with unit roundoff
\(u_{\rm s}\). The verifier and all outward endpoints use binary64,
\(\mathbb F_{\rm v}=\mathbb F_{64}\), with \(u_{\rm v}=2^{-53}\). Stored FP32
statistics are promoted exactly to binary64 before verification; all other
maintained proof state uses binary64.

\paragraph{Outward evaluation.}
Define
\(\gamma_m^{(q)}=m u_q/(1-m u_q)\) for \(m u_q<1\) and
\(\bar u_q=u_q/(1-u_q)\), where \(q\in\{\mathrm{s},\mathrm{v}\}\).
For verifier arithmetic, we write
\(\gamma_m:=\gamma_m^{(\mathrm v)}\) and \(\bar u:=\bar u_{\mathrm v}\).
Let \(\eta=2^{-1074}\), and let \(\omega\) be the largest finite binary64
number.  For \(z\in\mathbb F_{64,\geq0}\), define
\[
 \operatorname{up}_m(z):=
 \operatorname{nextafter}\!\left(
 \frac{z}{1-\gamma_m^{(\mathrm v)}},+\infty\right).
\]
The shorthand \(\operatorname{up}(\cdot)\) suppresses only this arity
subscript: the implementation passes it explicitly, and each displayed use
has the unique arity prescribed by the immediately enclosed operation.
\(\operatorname{down}(z)=
\operatorname{nextafter}(\fl(z),-\infty)\) for its lower endpoint.
\(m=\ell\) for a sum of \(\ell\) nonnegative terms, \(m=1\) for one
product or division, and \(m=2p+2\) for a scale-safe Frobenius-norm
reduction over \(p\) entries.

\paragraph{Arithmetic scope.}
The dot-product bounds permit any parenthesization of binary32 or binary64
products and sums within the stated operation count, so the proof
does not depend on a particular threaded-BLAS reduction order. Fused
multiply--add is covered conservatively because it incurs no more rounding
than the separate product and addition modeled by the bound. Algebraic
reassociation through \texttt{fastmath} is disabled. TF32 and kernels with
mixed-precision accumulation are outside the certified arithmetic model.
Certified runs require IEEE gradual underflow; overflow, nonfinite values,
unsafe subnormal intermediates, or an endpoint beyond \(\omega\) returns
\textsc{Inconclusive}. Relative-error bounds without an explicit \(\eta\)
term are used only after their intermediates are proved to remain in the
normal range.
These binary64 conventions govern the ridge inverse realization.

\subsection{Retain-Set Statistics Proof State}
\label{app:statistics-error-invariant}

\paragraph{Initialization.}

At initialization, let
\(\widehat A_H=\fl_{\rm s}(\abs{X_0}^{\top}\abs{X_0})\) and
\(\widehat A_B=\fl_{\rm s}(\abs{X_0}^{\top}\abs{Y_0})\).  Reverse
dot-product bounds instantiate the enclosure below. The declared ridge is
interpreted as the exact real value of its stored encoding; the displayed
rounding term covers its addition to the stored diagonal.
\begin{equation}
\begin{aligned}
\mathcal U_H^{(0)}
&=\operatorname{up}\!\left(
\frac{\gamma_{n_0}^{(\mathrm s)}}{1-\gamma_{n_0}^{(\mathrm s)}}
\widehat A_H+\bar u_{\rm s}\abs{\widetilde H_0}\right)
\in\R_{\geq0}^{d\times d},\\
\mathcal U_B^{(0)}
&=\operatorname{up}\!\left(
\frac{\gamma_{n_0}^{(\mathrm s)}}{1-\gamma_{n_0}^{(\mathrm s)}}
\widehat A_B\right)
\in\R_{\geq0}^{d\times c},
\end{aligned}
\label{eq:statistics-initialization}
\end{equation}

\paragraph{Shared initialization across precisions.}
When two storage precisions share one binary64 initialization, the
lower-precision path instead uses the direct cast enclosure
\[
 \mathcal U_{H,{\rm s}}^{(0)}
 =\operatorname{up}\!\left(
   \mathcal U_{H,{\rm v}}^{(0)}
   +\abs{\operatorname{cast}_{\rm s}(\widetilde H_{0,{\rm v}})
          -\widetilde H_{0,{\rm v}}}\right),
\]
and analogously for \(B\).  This preserves a common exact data target while
isolating the precision used to maintain the stored statistics. We use this
shared initialization only for the matched FP32/FP64 diagnostic in
\cref{app:sequential-proof-state-behavior}, so the two arms differ in
subsequent storage precision rather than in their exact data target.

\paragraph{Signed rank-one updates.}
For an update sign \(s\in\{-1,+1\}\), let
\(\widehat G_H:=\fl_{\rm s}(xx^\top)\in\mathbb F_{\rm s}^{d\times d}\)
and
\(\widehat G_B:=\fl_{\rm s}(xy^\top)\in\mathbb F_{\rm s}^{d\times c}\).
The row entries \(x,y\) are values in \(\mathbb F_{\rm s}\), interpreted
exactly as reals.  The implementation sets
\[
\widetilde H^+:=\fl_{\rm s}(\widetilde H+s\widehat G_H),\qquad
\widetilde B^+:=\fl_{\rm s}(\widetilde B+s\widehat G_B)
\]
and updates
\begin{align}
\mathcal U_H^+
&:=\operatorname{up}\!\left(
\mathcal U_H+\bar u_{\rm s}\abs{\widehat G_H}
+\bar u_{\rm s}\abs{\widetilde H^+}\right),
\label{eq:statistics_H_bound_update}\\
\mathcal U_B^+
&:=\operatorname{up}\!\left(
\mathcal U_B+\bar u_{\rm s}\abs{\widehat G_B}
+\bar u_{\rm s}\abs{\widetilde B^+}\right).
\label{eq:statistics_B_bound_update}
\end{align}

Each update adds the local outer-product and storage-rounding errors to the
existing statistics enclosure. The following theorem establishes that this
enclosure remains valid after every committed update.

\begin{theorem}[Statistics-error-bound soundness]
\label{thm:statistics-error-bound}
Suppose the current state satisfies
\(\abs{\widetilde H-H}\leq\mathcal U_H\) and
\(\abs{\widetilde B-B}\leq\mathcal U_B\).
For \(s\in\{-1,+1\}\), let
\(H^+=H+sxx^\top\) and \(B^+=B+sxy^\top\).
Whenever the certified arithmetic-scope guards pass, the updates in
\cref{eq:statistics_H_bound_update,eq:statistics_B_bound_update} satisfy
\[
\abs{\widetilde H^+-H^+}\leq\mathcal U_H^+,\qquad
\abs{\widetilde B^+-B^+}\leq\mathcal U_B^+.
\]
The initialization in \cref{eq:statistics-initialization} also encloses the
exact statistics represented by \(X_0,Y_0\) and the declared ridge encoding.
A failed guard makes the transition noncommittable.
\end{theorem}

\begin{proofbox}[Proof of statistics-error-bound soundness]
For \(H\), inserting the exact outer product gives
\[
\begin{aligned}
\abs{\widetilde H^+-H^+}
&\leq\abs{\widetilde H-H}
 +\abs{\widehat G_H-xx^\top}\\
&\quad+\abs{\widetilde H^+
-(\widetilde H+s\widehat G_H)}.
\end{aligned}
\]
The standard componentwise models of \citet[Ch.~3]{higham2002} bound the last
two terms by
\(\bar u_{\rm s}\abs{\widehat G_H}\) and
\(\bar u_{\rm s}\abs{\widetilde H^+}\), respectively. Combining them with
the previous enclosure proves \cref{eq:statistics_H_bound_update}; the same
argument applies to \(B\). The standard dot-product model gives the
initialization radii, and induction establishes
\cref{eq:stored-retained-budgets} after every committed update.
\end{proofbox}

\subsection{Exact-Dyadic Retain-Set-Statistics Rebuild}
\label{app:exact_rebuild}

A retained-data rebuild constructs statistics enclosures independently of
accumulated update errors. For each statistic entry, integer accumulation
computes the exact dyadic sum \(\xi\). The procedure stores
\(\widetilde\xi=\operatorname{RN}_{64}(\xi)\) and rounds
\(\abs{\widetilde\xi-\xi}\) outward to obtain the enclosure radius. Any
nonfinite conversion returns \textsc{Inconclusive}. Binary floating-point
products are dyadic, so integer accumulation is exact before the single final
rounding. The resulting entrywise radii therefore satisfy the initialization
hypotheses of \cref{thm:statistics-error-bound}; successful inverse-state
initialization and recertification then permit release under
\cref{thm:policy}.

We distinguish three event types in the implementation records. An
\emph{evidence rebuild} replaces only inverse or factor proof state. A
\emph{verifier-statistics reconstruction} replaces the verifier's retained
statistics and their enclosures. Neither event changes the candidate, and the
producer does not read the rebuilt arrays. A \emph{producer remediation}
instead changes producer state or recomputes the output; it creates a new
candidate with a new hash and may alter subsequent outputs, so it is recorded
and certified separately. Same-output comparisons include only the first two
event types.

Reconstruction visits \(O(n_td^2+n_tdc)\) dot-product terms; subsequent
proof-state initialization and certification have the costs in
\cref{app:transactional-safety-proof}.

\subsection{Stored-to-Retain-Set Transfer}
\label{app:joint_bounds}

The retain-set statistics enclosure from
\cref{app:statistics-error-invariant} and the propagated defect enclosure from
\cref{app:sequential-inverse} provide the executable \((\delta,\rho)\)
required by \cref{thm:joint}.  Let
\(\widehat R=\fl(\widetilde B-\widetilde H\widetilde\theta)\) and define
\begin{equation}
\mathcal U_R
:=\operatorname{up}\!\left(
  \gamma_d\abs{\widetilde H}\abs{\widetilde\theta}
  +\bar u\abs{\widehat R}+\mathcal U_B+
  \mathcal U_H\abs{\widetilde\theta}\right).
\label{eq:retained-residual-envelope}
\end{equation}
The first two terms bound floating-point evaluation of the stored residual.
The last two account for the difference between the stored statistics and the
exact retain-set statistics.
The envelope in \cref{eq:retained-residual-envelope}, the standard
matrix-product and subtraction models, and
\cref{eq:stored-retained-budgets,eq:reference-transfer-identity} give
\[
\abs{B-H\widetilde\theta-\widehat R}\leq\mathcal U_R .
\]
With
\(\delta_{\rm stored}=\operatorname{up}(\norm{\widehat E}_F+r)\), write
\(\overline{\norm{A}_2}\) for the quadratic-cost upper bound
\[
 \operatorname{up}\!\left(
 \sqrt{\operatorname{up}(\norm A_1)
              \operatorname{up}(\norm A_\infty)}\right).
\]
We next combine this residual enclosure with the stored inverse-defect bound
to obtain the two hypotheses required by \cref{thm:joint}:
\begin{equation}
\begin{aligned}
 \delta_{\rm ret}
 &=\operatorname{up}\!\left(
   \delta_{\rm stored}
   +\overline{\norm P_2}\,\overline{\norm{\mathcal U_H}_2}\right),
 \\
 \rho_{\rm ret}
 &=\operatorname{up}\!\left(
   \norm{\fl(P\widehat R)}_F\right.\\[-2pt]
 &\qquad\left.{}+
   \norm{\operatorname{up}(\gamma_d\abs P\abs{\widehat R})}_F\right.\\[-2pt]
 &\qquad\left.{}+
   \norm{\operatorname{up}(\abs P\mathcal U_R)}_F\right).
\end{aligned}
\label{eq:transferred-joint-bounds}
\end{equation}
\begin{proofbox}[Proof of the executable joint bounds]
The matrix-product model and the residual budget give
\[
\begin{aligned}
\norm{P(B-H\widetilde\theta)}_F
&\leq \norm{P\widehat R}_F+
       \norm{\abs{P}\mathcal U_R}_F\\
&\leq \norm{\fl(P\widehat R)}_F+
       \norm{\operatorname{up}(
         \gamma_d\abs{P}\abs{\widehat R})}_F\\
&\quad+\norm{\operatorname{up}(
         \abs{P}\mathcal U_R)}_F
\leq \rho_{\rm ret},
\end{aligned}
\]
which proves the residual hypothesis in \cref{eq:delta-rho-contract}.

For the defect, suppose \(r\in\R_{\geq0}\) satisfies
\(\norm{E-\widehat E}_F\leq r\).  Because
\(I_d-PH=(I_d-P\widetilde H)+P(\widetilde H-H)\),
\[
\begin{aligned}
\norm{I_d-PH}_2
&\leq \norm{\widehat E}_F+r+
\norm{P}_2\norm{\widetilde H-H}_2\\
&\leq \norm{\widehat E}_F+r
 +\overline{\norm P_2}\,\overline{\norm{\mathcal U_H}_2}
\leq\delta_{\rm ret}.
\end{aligned}
\]
The second line follows from
\(\norm A_2\leq\sqrt{\norm A_1\norm A_\infty}\) and
\(\abs{\widetilde H-H}\leq\mathcal U_H\), which implies
\(\norm{\widetilde H-H}_1\leq\norm{\mathcal U_H}_1\) and likewise for
the infinity norm. The final inequalities are the two bounds in
\cref{eq:transferred-joint-bounds}, which supply both hypotheses of
\cref{thm:joint}.
The row and column sums used by the two induced-norm bounds cost \(O(d^2)\); the
remaining residual products cost \(O(d^2c)\).  No dense matrix--matrix
product is evaluated on this release path.
\end{proofbox}

Retaining componentwise arrays preserves the directional residual interaction
\(\abs P\mathcal U_R\), while the defect transfer avoids forming the dense
matrix product \(\abs P\mathcal U_H\).

\subsection{Executable Release Certificates}
\label{app:release-certificates}

\subsubsection{Parameter Certificate}

Substituting the transferred bounds into \cref{thm:joint} gives the absolute
parameter certificate, provided the bounds are finite and
\(\delta_{\rm ret}<1\). For the relative contract, let \(p=dc\), compute the
scale-safe floating-point norm \(\widehat s=\fl(\norm{\widetilde\theta}_F)\),
and form the certified lower endpoint
\[
 s^-:=\operatorname{down}\!\left(
       \frac{\widehat s}{1+\gamma_{2p+2}}
      \right),
 \qquad
 d^-:=\operatorname{down}(s^--C).
\]
The verifier requires \(d^->0\), then computes
\(\tau_{\rm rel}^+:=\operatorname{up}(C/d^-)\) and certifies only when
\(\tau_{\rm rel}^+\leq\tau_{\rm par}\). Thus the norm lower bound, denominator
subtraction, and final division are all directed outward.

\subsubsection{Fixed-Readout Error Center}
\label{app:fixed-readout-executable}

Let \(\Delta=\theta^{\mathrm R}-\widetilde\theta\). The corrected residual
provides an approximate center for this parameter error. Applying the fixed
readout to the center preserves information about the error direction before
the remaining uncertainty is bounded. Define
\[
\begin{aligned}
 \widehat\Delta&=\fl(P\widehat R),\\
 \mathcal U_c^{\rm fl}
  &=\operatorname{up}(\gamma_d\abs P\abs{\widehat R}),\\
 \mathcal U_c^{\rm tr}
  &=\operatorname{up}(\abs P\mathcal U_R).
\end{aligned}
\]
Their sum encloses
\(\widehat\Delta-P(B-H\widetilde\theta)\), so the executable
center-evaluation radius is
\[
\begin{aligned}
 \eta_c&=\operatorname{up}\!\left(
   \norm{\mathcal U_c^{\rm fl}}_F+\norm{\mathcal U_c^{\rm tr}}_F\right),\\
 R_\Delta&=\operatorname{up}(\eta_c+\delta_{\rm ret}C).
\end{aligned}
\]
The inverse-defect bound gives
\(\norm{\Delta-P(B-H\widetilde\theta)}_F\leq\delta_{\rm ret}C\);
therefore \(\norm{\Delta-\widehat\Delta}_F\leq R_\Delta\), as required by
\cref{prop:fixed-readout-error-center}.
For a fixed readout \(Z\), compute
\(\widehat y=\fl(Z\widehat\Delta)\) and
\[
 \mathcal U_y=\operatorname{up}(\gamma_d\abs Z\abs{\widehat\Delta}).
\]
The two sound readout bounds used by the implementation are
\begin{align*}
 C_Z^{\rm ball}
  &=\operatorname{up}(\overline{\norm Z_2}\,C),\\
 C_Z^{\rm ctr}
  &=\operatorname{up}\!\left(
    \norm{\widehat y}_F+\norm{\mathcal U_y}_F
    +\overline{\norm Z_2}\,R_\Delta\right).
\end{align*}
The fixed readout is preprocessed once with the rigorous bound
\[
 \overline{\norm Z_2}
 =\min\!\left\{
   \operatorname{up}(\norm Z_F),
   \operatorname{up}\!\left(
      \sqrt{u_1u_\infty}
   \right)
 \right\}.
\]
Here \(u_1=\operatorname{up}(\norm{Z}_1)\) and
\(u_\infty=\operatorname{up}(\norm{Z}_\infty)\).
Both terms upper-bound \(\norm Z_2\), so their minimum remains rigorous.
The verifier returns
\(\min\{C_Z^{\rm ball},C_Z^{\rm ctr}\}/\sqrt{mc}\) for the normalized
contract in \cref{eq:fixed-readout-contract}. Every displayed operation is
evaluated outward, and the range and finiteness guards from
\cref{app:fp-conventions} apply before the minimum is used.
For \(m\) fixed readout rows, the added work is \(O(mdc)\); the center uses
\(O(dc)\) state, while storing \(Z\) requires \(O(md)\) entries unless the
service already retains the readout.

\subsection{Sequential Inverse Proof State}
\label{app:sequential-inverse}

\subsubsection{State invariant and strict initialization}
\label[appendix]{app:tracker_invariant}

The maintained inverse state consists of \(P\), a computed defect tracker
\(\widehat E\), and a certified radius \(r\) satisfying
\[
 \norm{I_d-P\widetilde H-\widehat E}_F\leq r.
\]
Strict initialization evaluates the defect directly and encloses its
arithmetic error.

Throughout this subsection, \(\widetilde H\) denotes the exact-real embedding
of the entries of the stored floating-point array, whether that array is held
in binary32 or binary64. Write \(H\) for the exact retain-set matrix. Thus the tracked
defect is \(E=I_d-P\widetilde H\), while the statistics-error bound later
transfers this bound to
\(I_d-PH\).

At strict initialization, let
\(\widehat E_0:=\fl(I_d-P_0\widetilde H_0)
\in\mathbb F_{64}^{d\times d}\), where
\(P_0,\widetilde H_0\in\mathbb F_{64}^{d\times d}\).  The standard product and
subtraction models give
\[
r_0:=\operatorname{up}\!\left(
\norm{\operatorname{up}(
\gamma_d\abs{P_0}\abs{\widetilde H_0}
+\bar u\abs{\widehat E_0})}_F
\right)\in\R_{\geq0}.
\]
It satisfies
\(\norm{I_d-P_0\widetilde H_0-\widehat E_0}_F\leq r_0\), establishing the base case
before any rank-one transition is attempted.

Each accepted transition adds three local error budgets to the previous
radius:
\[
 r^+=\operatorname{up}\!\left(
 r+\eta_{\mathrm{alg}}+\eta_{\mathrm{state}}
 +\eta_{\mathrm{fl}}\right).
\]
Here \(\eta_{\rm alg}\) encloses the use of computed low-rank intermediates,
\(\eta_{\rm state}\) the storage error in the updated arrays, and
\(\eta_{\rm fl}\) the evaluation error in the defect tracker.

\subsubsection{Local arithmetic bounds}

Before an update, let \(P,\widetilde H\in\mathbb F_{64}^{d\times d}\),
\(x\in\mathbb F_{64}^{d}\), \(s\in\{-1,+1\}\), and
\(E:=I_d-P\widetilde H\), with stored values embedded exactly in the reals.
Define
\[
\begin{aligned}
q&=Px,& z&=P^\top x,&
\beta&=x^\top q=z^\top x,& \alpha&=1+s\beta,\\
\widehat q&=\fl(Px),& \widehat z&=\fl(P^\top x),&
\widehat\beta&=\fl(x^\top\widehat q),&
\widehat\alpha&=\fl(1+s\widehat\beta).
\end{aligned}
\]
For comparison with the computed intermediates, set
\[
 \beta^\circ:=x^\top\widehat q,
 \qquad
 \alpha^\circ:=1+s\beta^\circ.
\]
The standard matrix--vector model gives
\[
\begin{aligned}
\varepsilon_q&:=\operatorname{up}(\gamma_d\abs P\abs x),&
\abs{q-\widehat q}&\leq \varepsilon_q,\\
\varepsilon_z&:=\operatorname{up}(\gamma_d\abs{P^\top}\abs x),&
\abs{z-\widehat z}&\leq \varepsilon_z.
\end{aligned}
\]

\paragraph{Certified denominator guard.}
\label[appendix]{app:denominator-safeguard}

The implementation forms
\[
\begin{aligned}
\varepsilon_{\alpha,\mathrm{fl}}
&:=\operatorname{up}\!\left(
\gamma_d\operatorname{up}(\abs{x}^{\top}\abs{\widehat q})
+\abs{\operatorname{TwoSum}(1,s\widehat\beta;\widehat\alpha)}\right),\\
\varepsilon_\alpha
&:=\operatorname{up}\!\left(
\varepsilon_{\alpha,\mathrm{fl}}+\abs{x}^{\top}\varepsilon_q\right),\\
\varepsilon_{qz}
&:=\operatorname{up}\!\left(
\abs{x}^{\top}(\varepsilon_q+\varepsilon_z)\right).
\end{aligned}
\]
The standard dot-product model and exact \textsc{TwoSum} residual give
\(\abs{\alpha-\widehat\alpha}\leq\varepsilon_\alpha\). Hence
\(\widehat\alpha>\varepsilon_\alpha\) certifies \(\alpha>0\) before any
division. An additional \(\sqrt u(1+\abs{\widehat\beta})\) policy floor sends
numerically fragile positive denominators to the rebuild path; it is not used
in the soundness argument.

The guard ensures \(\widehat\alpha\ne0\). Define the exact-real matrices formed
from the computed low-rank intermediates,
\[
\begin{aligned}
P^{+\circ}&:=P-\frac{s\widehat q\widehat z^\top}{\widehat\alpha}
, &
H^{+\circ}&:=\widetilde H+sxx^\top, &
E^{+\circ}&:=I_d-P^{+\circ}H^{+\circ}.
\end{aligned}
\]
The implementation stores \(P^+\) in binary64 and \(\widetilde H^+\) in
\(\mathbb F_{\rm s}\), promoting the latter exactly before tracker evaluation.
With both arrays embedded in the reals, let
\(E_{\mathrm{store}}^+:=I_d-P^+\widetilde H^+\) and define
\[
\mathcal T_s^\circ(Z)
:=Z-\widehat q\frac{s}{\widehat\alpha}
  \left(x^\top-\widehat z^\top\widetilde H\right),
\qquad Z\in\R^{d\times d}.
\]
For this map, \cref{thm:radius} requires
\[
\begin{aligned}
\norm{E-\widehat E}_F&\leq r,\\
\norm{E^{+\circ}-\mathcal T_s^\circ(E)}_F
&\leq\eta_{\mathrm{alg}},\\
\norm{E_{\mathrm{store}}^+-E^{+\circ}}_F
&\leq\eta_{\mathrm{state}},\\
\norm{\mathcal T_s^\circ(\widehat E)-\widehat E^+}_F
&\leq\eta_{\mathrm{fl}}.
\end{aligned}
\]
The constructions below bound the final three terms.

\begin{proofbox}[Computed-intermediate mismatch bound]
The exact vectors satisfy
\[
x^\top z=x^\top P^\top x=(x^\top Px)^\top=x^\top q,
\]
because both sides are real scalars.  Hence
\[
\begin{aligned}
\abs{x^\top(\widehat z-\widehat q)}
&=\abs{x^\top(\widehat z-z)+x^\top(q-\widehat q)}\\
&\leq\abs{x}^{\top}(\varepsilon_z+\varepsilon_q)
\leq \varepsilon_{qz}.
\end{aligned}
\]
Expanding \(E^{+\circ}:=I_d-P^{+\circ}H^{+\circ}\) gives the exact identity
\begin{equation}
\begin{aligned}
&E^{+\circ}
-\mathcal T_s^\circ(E)\\
&=-s(q-\widehat q)x^\top
+\frac{\widehat q}{\widehat\alpha}
\left[
s(\alpha^\circ-\widehat\alpha)
+x^\top(\widehat z-\widehat q)
\right]x^\top .
\end{aligned}
\label{eq:computed_intermediate_mismatch}
\end{equation}
To verify the coefficient of \(\widehat qx^\top\), observe that
\[
-s+\frac{s+x^\top\widehat z}{\widehat\alpha}
=\frac{
s(\alpha^\circ-\widehat\alpha)
+x^\top(\widehat z-\widehat q)}
{\widehat\alpha},
\]
where \(\alpha^\circ=1+s x^\top\widehat q\).
The componentwise budgets and
\(\|ab^\top\|_F=\|a\|_2\|b\|_2\) now imply
\[
\begin{aligned}
\eta_{\mathrm{alg}}
:=\operatorname{up}\!\Bigl(
&\norm{\varepsilon_q}_2\norm{x}_2\\
&+
\frac{\varepsilon_{\alpha,\mathrm{fl}}+\varepsilon_{qz}}
{\abs{\widehat\alpha}}
\norm{\widehat q}_2\norm{x}_2\Bigr)
\in\R_{\geq0}.
\end{aligned}
\]
This quantity upper-bounds the Frobenius norm of
\cref{eq:computed_intermediate_mismatch}.
\end{proofbox}

\paragraph{Stored-state update error.}

Let \(\Delta P:=-s\widehat q\widehat z^\top/\widehat\alpha\),
\(\Delta H:=sxx^\top\), \(\widehat{\Delta P}:=\fl_{64}(\Delta P)\), and, for
binary64 statistics, \(\widehat{\Delta H}:=s\fl_{64}(xx^\top)\). The standard
models in \cref{app:fp-conventions}, including the exact componentwise
\textsc{TwoSum} residuals
\(S_P=(P+\widehat{\Delta P})-P^+\) and
\(S_H=(\widetilde H+\widehat{\Delta H})-\widetilde H^+\), give
\[
\begin{aligned}
\mathcal U_{\Delta P}
&:=\operatorname{up}\!\left(
  \bar\gamma_2\abs{\widehat{\Delta P}}+\abs{S_P}\right),&
\abs{P^+-(P+\Delta P)}&\leq\mathcal U_{\Delta P},\\
\mathcal U_{\Delta H}
&:=\operatorname{up}\!\left(
  \bar u\abs{\widehat{\Delta H}}+\abs{S_H}\right),&
\abs{\widetilde H^+-(\widetilde H+\Delta H)}
&\leq\mathcal U_{\Delta H},
\end{aligned}
\]
where \(\bar\gamma_2:=\gamma_2/(1-\gamma_2)\).
For binary32 statistics, the same quantity is obtained by promoting the
post-update array exactly and enclosing its difference from the nominal update.
A product-difference expansion then gives
\[
\begin{aligned}
\eta_{\mathrm{state}}
:=\operatorname{up}\!\big(&
\norm{\mathcal U_{\Delta P}}_F
  (\norm{\widetilde H}_F+\norm{\Delta H}_F)\\
&+(\norm{P}_F+\norm{\Delta P}_F)
  \norm{\mathcal U_{\Delta H}}_F\\
&+\norm{\mathcal U_{\Delta P}}_F
  \norm{\mathcal U_{\Delta H}}_F\big)
\geq\norm{E_{\mathrm{store}}^+-E^{+\circ}}_F.
\end{aligned}
\]

\paragraph{Tracker-evaluation error.}

The tracker evaluates
\[
\begin{aligned}
\widehat w&:=\fl(\widetilde H^\top\widehat z),&
\widehat v&:=\fl(x-\widehat w),\\
\widehat a&:=\fl((s/\widehat\alpha)\widehat v),&
\widehat C&:=\fl(\widehat q\widehat a^\top).
\end{aligned}
\]
The standard componentwise models give
\[
\begin{aligned}
\mathcal U_w&:=\operatorname{up}\!\left(
 \gamma_d\abs{\widetilde H}^{\top}\abs{\widehat z}\right),&
\mathcal U_v&:=\operatorname{up}\!\left(
 \mathcal U_w+\bar u\abs{\widehat v}\right),\\
\mathcal U_a&:=\operatorname{up}\!\left(
 \mathcal U_v/\abs{\widehat\alpha}
 +\bar\gamma_2\abs{\widehat a}\right),&
\mathcal U_C&:=\operatorname{up}\!\left(
 \abs{\widehat q}\mathcal U_a^\top
 +\bar u\abs{\widehat C}\right).
\end{aligned}
\]
Finally, let
\(S_E:=(\widehat E-\widehat C)-\widehat E^+\) be the exact
\textsc{TwoSum} residual
for \(\widehat E^+:=\fl(\widehat E-\widehat C)\), and set
\(\mathcal U_S:=\abs{S_E}\). Therefore
\[
\eta_{\mathrm{fl}}
:=\operatorname{up}\!\left(
\norm{\mathcal U_C+\mathcal U_S}_F\right)
\geq
\norm{\mathcal T_s^\circ(\widehat E)-\widehat E^+}_F.
\]

\newpage
\subsubsection{Radius propagation theorem}
\label[appendix]{app:radius-proof}

\begin{samepage}
\begin{theorem}[Sequential proof-state invariant]
\label{thm:radius}
Suppose \(\norm{E-\widehat E}_F\leq r\).  Let
\(\eta_{\rm alg}\), \(\eta_{\rm state}\), and
\(\eta_{\rm fl}\) satisfy
\[
\begin{aligned}
\norm{E^{+\circ}-\mathcal T_s^\circ(E)}_F
&\leq\eta_{\rm alg},\\
\norm{E_{\rm store}^+-E^{+\circ}}_F
&\leq\eta_{\rm state},\\
\norm{\mathcal T_s^\circ(\widehat E)-\widehat E^+}_F
&\leq\eta_{\rm fl}.
\end{aligned}
\]
Then
\[
 r^+=\operatorname{up}\!\left(
 r+\eta_{\rm alg}+\eta_{\rm state}+\eta_{\rm fl}
 \right)
\]
satisfies
\[
\norm{E_{\rm store}^+-\widehat E^+}_F\leq r^+ .
\]
All three local bounds can be evaluated in \(O(d^2)\) work per rank-one
update.
\end{theorem}

\begin{proofbox}[Proof of propagated-radius soundness]
The map \(\mathcal T_s^\circ\) has identity linear part, so
\(\mathcal T_s^\circ(E)-\mathcal T_s^\circ(\widehat E)=E-\widehat E\).
Inserting the nominal exact update and applying the triangle inequality gives
\[
\begin{aligned}
\norm{E_{\mathrm{store}}^+-\widehat E^+}_F
&\leq
\norm{E_{\mathrm{store}}^+-E^{+\circ}}_F\\
&\quad+\norm{E^{+\circ}-\mathcal T_s^\circ(E)}_F\\
&\quad+\norm{E-\widehat E}_F
+\norm{\mathcal T_s^\circ(\widehat E)-\widehat E^+}_F\\
&\leq r+\eta_{\mathrm{state}}+\eta_{\mathrm{alg}}+\eta_{\mathrm{fl}}
\leq r^+.
\end{aligned}
\]
Strict initialization gives the base case; induction proves the invariant for
every guarded tracker update.
\end{proofbox}

\end{samepage}

\section{Published-Method Applications}
\label{app:published-method-applications}

Candidates are fixed before post-hoc certification. Independent reference
intervals supply offline conformance labels and are excluded from online
timing. Service protocols and reference construction appear in
\cref{app:service-evaluation,app:reference-reproducibility}.

\subsection{Application Details}
\label{app:published-output-applications}

We use pinned official implementations and certify emitted candidates post hoc
without changing their update rules. Exact source revisions, instrumentation
patches, and reproduction commands are recorded in the artifact.

\paragraph{Newton-update removal.}
We use the official MNIST 3-vs-8 configuration of \citet{guo2020}, leaving its
objective perturbation, update, cumulative budget, and threshold unchanged.
Native closure encloses the source terms through the first retraining action,
and direct recertification checks the same candidate at that stop. Strict
closure can move retraining earlier, while the tighter direct certificate can
expose a conservative source action.

\paragraph{CEU supported-fidelity characterization.}
Post-update \RRV{} evaluates the emitted CEU head under an independently
declared retraining-reference contract~\citep{wu2023ceu}. We fixed
\(\tau_{\rm par}=10^{-3}\) before inspecting the outputs. The reference is the
unique minimizer of retained-node mean softmax cross-entropy plus
\(\lambda\norm{\theta}_2^2/2\), giving the curvature floor
\(\mu^-=\lambda\) in \cref{prop:strongly-convex-post-update}. Directed
arithmetic encloses the retained-objective gradient; verified optimization is
used only for offline tightness evaluation. The certificates cover their
reference intervals and quantify the weaker fidelity supported by the
snapshots alongside CEU's native randomized guarantee.

\paragraph{ScaleGUN gradient application.}
Native closure evaluates the same aggregate gradient quantity consumed by
ScaleGUN's propagation and optimization budget~\citep{yi2025scalegun}. The
candidate is unchanged during certification.

\subsection{Newton Direct-Perturbation Calibration}
\label{app:theorem-input-audits}

We follow the Newton direct-perturbation mechanism of
\citet{zou2025highdim} in its published synthetic logistic-ridge configuration,
evaluating one- and two-step Newton centers over 100 leave-one-out states.
Native closure bounds displacement from the exact-real Newton center, while
\OnDemandRRV{} certifies the executed center against a verified retained-set
reference.

To separate source slack from direction-free composition, let
\(p\in\{32,64\}\) denote the execution precision and define
\[
 \begin{aligned}
A&=\frac{\norm{\bar\theta-\theta^{\mathrm R}}_2}{r_{\rm nom}},
&B(p)&=\frac{\norm{\theta_p-\bar\theta}_2}{r_{\rm nom}},\\
C(p)&=\frac{\norm{\theta_p-\theta^{\mathrm R}}_2}{r_{\rm nom}}.
 \end{aligned}
\]
Directed enclosures produce \(A^+,B^+(p),C^+(p)\) for every paired state
(\cref{tab:published-method-application-details}). We form each reported ratio
within a state before aggregating. For FP64, \(C^+(64)\) is near zero in some
states, so we report the absolute composition--direct gap rather than an
unstable ratio. Direction-free composition remains looser under both execution
precisions on this population.

\begin{table}[H]
  \caption{\textbf{Detailed published-method application results.}
  Newton-perturbation bounds are normalized by the nominal radius; ScaleGUN
  ratios report the source budget consumed. Newton ratios are formed per state
  before aggregation; FP64 uses an absolute gap because its direct bound can be
  near zero.}
  \label{tab:published-method-application-details}
  \centering
  \footnotesize
  \setlength{\tabcolsep}{3pt}
  \renewcommand{\arraystretch}{1.05}

  \begin{minipage}[t]{0.485\textwidth}
    \vspace{0pt}
    \noindent
    \textbf{(a) Newton removal}

    \vspace{1pt}
    \hrule
    \vspace{2pt}

    \begin{tabular*}{\linewidth}{@{\extracolsep{\fill}}crrrr@{}}
      Seed & Native & \ExecCert{} & $\Delta$ & Direct upper \\
      0 & 251 & 249 & $-2$ & .00304 \\
      1 & 232 & 232 & $0$ & .00324 \\
      2 & 277 & 273 & $-4$ & .00421 \\
    \end{tabular*}
  \end{minipage}
  \hfill
  \begin{minipage}[t]{0.485\textwidth}
    \vspace{0pt}
    \noindent
    \textbf{(b) Newton direct perturbation}

    \vspace{1pt}
    \hrule
    \vspace{2pt}

    \begin{tabular*}{\linewidth}{@{\extracolsep{\fill}}lrrr@{}}
      Bound & Median & P95 & Max \\
      Source center, $A^+$ & .087 & .341 & .945 \\
      FP32 $B^+$ & .175 & .181 & .186 \\
      FP32 $C^+$ & .111 & .291 & .874 \\
      FP32 comp./direct & \textbf{2.28$\times$} & 2.49$\times$ & 2.56$\times$ \\
      FP64 $B^+$ & .067 & .070 & .074 \\
      FP64 $C^+$ & .020 & .271 & .871 \\
      FP64 comp.--direct & .134 & .141 & .148 \\
    \end{tabular*}
  \end{minipage}

  \vspace{8pt}

  \begin{minipage}[t]{0.485\textwidth}
    \vspace{0pt}
    \noindent
    \textbf{(c) ScaleGUN}

    \vspace{1pt}
    \hrule
    \vspace{2pt}

    \begin{tabular*}{\linewidth}{@{\extracolsep{\fill}}lrrrr@{}}
      Dataset & Cheap & Heads & Median & Max \\
      Cora & 12 & 84 & .59\% & .83\% \\
      ogbn-arxiv & 3 & 120 & 1.24\% & 1.28\% \\
      All & 15 & 204 & .68\% & 1.28\% \\
    \end{tabular*}

    \vspace{1pt}
    {\footnotesize 9 additional checkpoints use the source fallback.}
  \end{minipage}
  \hfill
  \begin{minipage}[t]{0.485\textwidth}
    \vspace{0pt}
    \noindent
    \textbf{(d) CEU}

    \vspace{1pt}
    \hrule
    \vspace{2pt}

    \begin{tabular*}{\linewidth}{@{\extracolsep{\fill}}lrr@{}}
      Fidelity & Median & Max \\
      Realized error & 8.86\% & 9.47\% \\
      \RRV{} upper & 14.06\% & 14.84\% \\
    \end{tabular*}

    \vspace{1pt}
    {\footnotesize Declared tolerance: $\tau_{\rm par}=0.10\%$.}
  \end{minipage}
\end{table}

\FloatBarrier

\section{Service Evaluation and Controls}
\label{app:service-evaluation}
\label{app:frozen-transactional}

\subsection{Experimental Systems and Common Protocols}
\label{app:llm-case-study}
\label{app:guarantee-matched-service}

\paragraph{3B safety service.}
The deployment service uses a frozen Llama-3.2-3B-Instruct
backbone~\citep{meta2024llama32} with a \(d=3073\) binary ridge head on
PKU-SafeRLHF-QA~\citep{ji2024pkusaferlhf}. We replay seven 1,000-request
deletion streams under a \(10^{-3}\) release tolerance and evaluate
fixed-readout contracts on 3,000 sealed representations.

\paragraph{Controlled diagnostics.}
The controlled experiments use DINOv2/CIFAR-10 and RoBERTa/AG-News multiclass
ridge heads. The tolerance study fixes 14,000 outputs, and the stress study
crosses three deletion streams, three conditioning levels, and two statistics
storage precisions, giving 18 trajectories. Independent retained-row
reconstruction supplies all offline reference labels.

\paragraph{Verification constructions.}
The guarantee-matched comparison holds the candidate, fixed readout, release
contract, and conformance labels constant while varying only how certificate
evidence is constructed. Three held-out streams at seven declared checkpoints
give 21 matched states. The six constructions are summarized below.

\begin{table}[H]
  \caption{Verification constructions in the guarantee-matched comparison.}
  \label{tab:verification-constructions}
  \centering
  \footnotesize
  \setlength{\tabcolsep}{3pt}
  \renewcommand{\arraystretch}{1.08}
  \begin{tabular}{@{}>{\raggedright\arraybackslash}p{0.34\textwidth}
      >{\centering\arraybackslash}p{0.12\textwidth}
      >{\centering\arraybackslash}p{0.12\textwidth}
      >{\centering\arraybackslash}p{0.12\textwidth}
      >{\centering\arraybackslash}p{0.12\textwidth}@{}}
    \toprule
    Construction & Build at release & Retained rows & Reuse inverse
      & Readout \\
    \midrule
    \OnDemandRRV{}: retained residual & \cmark & \cmark & \xmark & Norm-ball \\
    \OnDemandRRV{}: corrected residual & \cmark & \cmark
      & \cmark\textsuperscript{a} & Error-centered \\
    \OnDemandRRV{}: fresh verified factor & \cmark & \xmark & \xmark & Error-centered \\
    Maintained verified factor & \xmark & \xmark & \xmark & Error-centered \\
    \IncrementalRRV{}: norm-ball & \xmark & \xmark & \cmark & Norm-ball \\
    \IncrementalRRV{}: directional & \xmark & \xmark & \cmark & Error-centered \\
    \bottomrule
  \end{tabular}
  \vspace{2pt}

  \raggedright
  \scriptsize\textsuperscript{a} Partial inverse reuse.
\end{table}

\subsection{Reference-Transfer Soundness and Tightness}
\label{app:frozen-head-protocol}

\paragraph{Tolerance--coverage frontier.}
For each of the 14,000 fixed outputs, the independent retain-set audit computes
the realized minimum tolerance \(\tau_{\rm true}\), while \RRV{} computes the
certified threshold \(\tau_{\rm cert}\). We sweep the deployment tolerance from
\(10^{-5}\) to \(10^{-2}\); every output satisfies
\(\tau_{\rm cert}\geq\tau_{\rm true}\).

\paragraph{Stored-equation stress comparison.}
This experiment tests whether certifying the stored equation is equivalent to
certifying the exact retain-set reference as statistics drift. The 18
trajectories defined above evaluate the stored equation and its certificate in
binary64 while varying statistics storage precision. Independent retained-row
reconstruction labels conformance at \(\tau_{\rm par}=10^{-3}\). Coverage and
false-release counts appear in \cref{tab:hard-stream-conformance}.

\begin{table}[H]
  \caption{\textbf{Matched-output conformance on nine workload configurations
  per precision.} Coverage is conditional on a conforming output. The
  incremental row reports the first check, before any proof-only rebuild. Four
  FP32 reference intervals remain unresolved and are not counted as conforming
  or nonconforming.}
  \label{tab:hard-stream-conformance}
  \centering
  \footnotesize
  \setlength{\tabcolsep}{4pt}
  \renewcommand{\arraystretch}{1.08}
  \begin{tabular}{@{}lrrrr@{}}
    \toprule
    Path & Total & Conforming & \shortstack{Certified\\conforming} & \shortstack{False\\release} \\
    \midrule
    FP32 stored equation & 9,000 & 7,293 & 7,293 & 1,703 \\
    FP32 + \IncrementalRRV{} & 9,000 & 7,293 & 4,838 & 0 \\
    FP64 + \IncrementalRRV{} & 9,000 & 9,000 & 9,000 & 0 \\
    \bottomrule
  \end{tabular}
\end{table}

\subsection{Evidence Reuse and Lifecycle Cost}

\paragraph{Guarantee-matched comparison.}
The \OnDemandRRV{} constructions form a retained-data residual, a corrected
residual, or a fresh verified factor when certification is requested. The
maintained-factor baseline instead downdates a factor after every deletion and
validates its residual at release. The two \IncrementalRRV{} constructions
share the inverse-defect proof state and differ only in their readout
certificate. The corrected-residual construction is the
only nonstandard baseline: it computes
\[
 \theta_{\rm corr}=\widetilde\theta
   +P(\widetilde B-\widetilde H\widetilde\theta)
\]
and evaluates a retained-data residual at \(\theta_{\rm corr}\). If this
residual gives
\(\norm{\theta_{\rm corr}-\theta^{\mathrm R}}_F\leq C_{\rm corr}\), it still
certifies the original candidate through
\[
 \norm{Z(\widetilde\theta-\theta^{\mathrm R})}_F
 \leq
 \norm{Z(\widetilde\theta-\theta_{\rm corr})}_F
 + \norm{Z}_2 C_{\rm corr}.
\]
All terms are evaluated outward, so every construction certifies the executed
snapshot rather than substituting \(\theta_{\rm corr}\) for release.

\begin{table}[H]
  \caption{Complete guarantee-matched results on 21 fixed states. Latency
  summarizes 10 repetitions per state; stored state excludes retained rows.}
  \label{tab:guarantee-matched-service-full}
  \centering\footnotesize
  \setlength{\tabcolsep}{4.5pt}
  \renewcommand{\arraystretch}{1.10}
  \begin{tabular}{@{}lcrrrr@{}}
    \toprule
    Verification method & Retain scan & \shortstack{Certified\\release states} & Med. (s) & P95 (s)
      & State (MiB) \\
    \midrule
    \OnDemandRRV{} (retained residual) & Yes & 0/21 & 5.163 & 5.410 & 72.1 \\
    \OnDemandRRV{} (corrected residual) & Yes & 21/21 & 5.341 & 5.696 & 286.5 \\
    \OnDemandRRV{} (fresh verified factor) & No & 21/21 & 2.611 & 3.593 & 286.5 \\
    Maintained verified factor & No & 21/21 & 2.453 & 3.678 & 358.6 \\
    \IncrementalRRV{} (norm-ball) & No & 0/21 & 0.801 & 0.859 & 360.3 \\
    \IncrementalRRV{} (directional) & No & 21/21 & 1.056 & 1.163 & 430.6 \\
    \bottomrule
  \end{tabular}
\end{table}

\paragraph{Decision agreement.}
Across 56 checkpoints from seven streams, margin certificates cover
\(95.0\)--\(95.8\%\) of declared decisions. The independent offline audit gives
\(99.96\%\) mean observed agreement rather than certificate coverage.

\paragraph{Release-frequency accounting.}
We measure output generation, certified-statistics maintenance, and proof-state
maintenance across all 1,000 requests in each of three streams. Release-check
cost is the per-stream median over the seven declared checkpoints, with 10
paired repetitions per checkpoint. We combine these measured components at 7,
25, 50, 100, 250, 500, and 1,000 checks per 1,000 requests. The three routes
share candidate-generation work; fresh factor pays no per-request proof
maintenance, while maintained factor and directional \IncrementalRRV{} do.
No route produces a false release on the matched states, and maintained factor
requires no rebuild during the stream replays. Linear component accounting
places the directional/fresh crossover at 21.5 checks and the maintained/fresh
crossover at 214.3 checks per 1,000 requests.

\subsection{Precision and Representation Boundaries}
\label{app:sequential-proof-state-behavior}

\paragraph{Producer-precision stress.}
Within each producer configuration, \IncrementalRRV{} and verified-factor
\OnDemandRRV{} inspect the same output; changing statistics precision defines
a separate candidate population. Lifecycle totals include all online work of
the selected path, while independent reference audits are excluded. The 4,162
lower-precision proof refreshes are triggered by a bound above tolerance, not
by invalid inverse state. The median bound moves from $1.014\tau$ before a
committed refresh to $0.830\tau$ afterward; every commit recovers
certification. All attempts together consume 4.39\% of the FP32 lifecycle and
no commit changes a candidate. \Cref{tab:hard-stream-conformance} reports the
first check, whereas \cref{tab:precision-boundary} reports the final decision
on the same population. Unresolved reference intervals are excluded from the
conforming denominator.

\begin{table}[H]
  \centering
  \caption{\textbf{Producer precision and certificate construction.}
  Rows change the emitted candidate; within each row, columns compare the same
  candidates. Coverage is conditional on conformance and uses the final
  post-rebuild decision; proof-only rebuilds do not change candidates.}
  \label{tab:precision-boundary}
  \small
  \setlength{\tabcolsep}{5pt}
  \begin{tabular}{@{}lrrr@{\qquad}rrr@{}}
    \toprule
    & \multicolumn{3}{c}{\OnDemandRRV{}}
    & \multicolumn{3}{c}{\IncrementalRRV{}} \\
    \cmidrule(lr){2-4}\cmidrule(l){5-7}
    Producer
      & Certified (\%) & Life (s) & State (MiB)
      & Certified (\%) & Life (s) & State (MiB) \\
    \midrule
    FP32 & \textbf{99.97} & \textbf{444.7} & \textbf{11.4}
         & 67.24 & 502.1 & 20.4 \\
    FP64 & \textbf{100.00} & \textbf{378.4} & \textbf{9.1}
         & \textbf{100.00} & 413.5 & 18.2 \\
    \bottomrule
  \end{tabular}
\end{table}

\paragraph{Contract--format compatibility frontier.}
\label{app:format-reachability}
This protocol asks whether the final representation can express an artifact
that satisfies the selected contract. We encode each format from the same
master head, exhaustively evaluate the BF16 and FP16 codebooks under
round-to-nearest-even, and rigorously enclose the minimum feasible tolerance
\(\tau_{\min}\) shown in
\cref{fig:deployment-tradeoffs}\textup{(b)}.

\FloatBarrier

\section{Independent Reference Construction}
\label{app:reference-reproducibility}

\subsection{Offline Verified Reference and Evaluation Labels}
\label{app:independent-reference}

Offline labels are reconstructed from retained rows, independently of the
service arrays and update path. A verified inverse-defect certificate encloses
the retraining solution, and a high-precision dyadic reconstruction confirms
all classifications. For its solution \(\theta_t^\star\), directed evaluation
of \(r_t=B_t-H_t\theta_t^\star\) gives
\[
 r_{\rm ref,t}=\frac{\norm{r_t}_F}{\lambda},
 \qquad
 \norm{\theta_t^\star-\theta_t^{\mathrm R}}_F\leq r_{\rm ref,t}.
\]

Directed norm bounds and the triangle inequalities of
\cref{eq:relative-parameter-contract} then give a rigorous relative-error
interval; precision is increased until each classification is resolved.

\subsection{Newton Direct-Perturbation Reference Construction}

To close the Newton direct-perturbation mechanism against its exact-real rule,
let \(s_j=\theta_j^{64}-\theta_{j-1}^{64}\) denote the exact-real difference
of two stored binary64 iterates. The verifier represents this dyadic difference
as the unevaluated sum returned by \textsc{TwoSum}; it does not identify
\(s_j\) with an ordinary rounded subtraction. This definition also handles a
solver step swallowed by the state addition: the realized step is then zero.
Let \(g(\theta)=\nabla L_{\rm ret}(\theta)\) and
\(H(\theta)=\nabla^2 L_{\rm ret}(\theta)\). On a convex region containing the
exact-real and replay iterates, and hence the segments connecting them, assume
\(mI\preceq H(\theta)\preceq MI\) and
\(\norm{H(\theta)-H(\theta')}_2\leq
L\norm{\theta-\theta'}_2\). We define the Newton-equation residual
\(r_j=H(\theta_{j-1}^{64})s_j+g(\theta_{j-1}^{64})\). Directed MPFR
evaluation applies the Hessian to both components of the exact difference and
bounds \(\norm{r_j}_2\). We propagate
\[
 \begin{aligned}
 \eta_j&=\eta_{j-1}
 +\frac{M\eta_{j-1}+L\eta_{j-1}\norm{s_j}_2+\norm{r_j}_2}{m},\\
 \eta_0&=0.
 \end{aligned}
\]
Thus \(\eta_j\) bounds the replay center's distance from the exact-real Newton
iterate, and \(\norm{\theta_j^{32}-\theta_j^{64}}_2+\eta_j\) closes the
execution error used by the compositional Laplace interface. This route needs
no retraining reference; the direct route instead encloses the candidate
against a verified retained-set reference.

\FloatBarrier

\end{document}